\documentclass[11pt,a4paper]{article}
\usepackage[utf8x]{inputenc}
\usepackage[T1]{fontenc}

\usepackage[pdftex]{graphicx} 
\usepackage[english]{babel} 
\usepackage[pdftex,linkcolor=black,pdfborder={0 0 0}]{hyperref} 
\usepackage{calc} 
\usepackage{enumitem} 

\usepackage[a4paper, lmargin=0.16\paperwidth, rmargin=0.16\paperwidth, tmargin=0.1\paperheight, bmargin=0.1\paperheight]{geometry} 

\usepackage[all]{nowidow} 
\usepackage[protrusion=true,expansion=true]{microtype} 

\usepackage{amsmath, amsfonts, amsthm, mathtools}
\newtheorem{theorem}{Theorem}
\newtheorem{lemma}{Lemma}
\newtheorem*{statement}{\bf Statement}

\allowdisplaybreaks

\DeclareMathOperator\erf{erf}
\DeclareMathOperator\TV{\mathsf{TV}}
\DeclareMathOperator\mean{mean}

\usepackage{gensymb}

\usepackage{graphicx}
\graphicspath{ {./Figures/} }

\usepackage{caption}
\usepackage{subcaption}

\usepackage[numbers, compress]{natbib}
\usepackage{fancyhdr}
\fancypagestyle{pgstyle}{\fancyhf{}\fancyfoot[L]{\small An updated version of this work titled `Provable Robustness against Wasserstein Distribution Shifts via Input Randomization' is available on OpenReview: \url{https://openreview.net/forum?id=HJFVrpCaGE}.}}

\begin{document} %

\title{Certifying Model Accuracy under Distribution Shifts}
\date{} 

\author{
    Aounon Kumar\\
    University of Maryland\\
    \texttt{aounon@umd.edu}\\
    \and
    Alexander Levine\\
    University of Maryland\\
    \texttt{alevine0@cs.umd.edu}\\
    \and
    Tom Goldstein\\
    University of Maryland\\
    \texttt{tomg@cs.umd.edu}\\
    \and
    Soheil Feizi\\
    University of Maryland\\
    \texttt{sfeizi@cs.umd.edu}\\
}

\maketitle
\thispagestyle{pgstyle}

\begin{abstract}
Certified robustness in machine learning has primarily focused on adversarial perturbations of the input with a fixed attack budget for each point in the data distribution. In this work, we present provable robustness guarantees on the accuracy of a model under bounded Wasserstein shifts of the data distribution. We show that a simple procedure that randomizes the input of the model within a transformation space is provably robust to distributional shifts under the transformation. Our framework allows the datum-specific perturbation size to vary across different points in the input distribution and is general enough to include fixed-sized perturbations as well. Our certificates produce guaranteed lower bounds on the performance of the model for any (natural or adversarial) shift of the input distribution within a Wasserstein ball around the original distribution. We apply our technique to: (i) certify robustness against natural (non-adversarial) transformations of images such as color shifts, hue shifts and changes in brightness and saturation, (ii) certify robustness against adversarial shifts of the input distribution, and (iii) show provable lower bounds (hardness results) on the performance of models trained on so-called ``unlearnable'' datasets that have been poisoned to interfere with model training.
Code for our experiments is available on GitHub: \url{https://github.com/aounon/distributional-robustness}.
\end{abstract}

\section{Introduction}
Machine learning models often suffer significant performance loss under minor shifts in the data distribution that do not affect a human's ability to perform the same task-- e.g., input noise \cite{DodgeK16, GeirhosTRSBW18}, image scaling, shifting and translation \cite{Azulay2019}, spatial \cite{EngstromTTSM19} and geometric transformations \cite{FawziF15, AlcornLGWMKN19}, blurring \cite{blur2016, ZhouSC17}, acoustic corruptions \cite{Pearce2000TheAE} and adversarial perturbations \cite{Szegedy2014, Carlini017, GoodfellowSS14, MadryMSTV18, BiggioCMNSLGR13}.
Overcoming such robustness challenges is a major hurdle for deploying these models in 
safety-critical applications where reliability is paramount.
Several training techniques have been developed to improve the empirical robustness of a model to data shifts, e.g., diversifying datasets \cite{TaoriDSCRS20}, training with natural corruptions \cite{HendrycksD19}, data augmentations \cite{YangWH19}, contrastive learning \cite{KimTH20, RadfordKHRGASAM21, robust_contrastive_learning} and adversarial training \cite{GoodfellowSS14, MadryMSTV18, TramerB19, ShafahiNG0DSDTG19, MainiWK20}.
These methods are designed to withstand specific ways of introducing changes in inputs (e.g., a particular adversarial attack procedure) and may break down under a different method for generating perturbations, e.g., adversarial defenses becoming ineffective under newer attacks \cite{Carlini017, athalye18a, UesatoOKO18, LaidlawF19, Laidlaw_perceptual}.

Certifiable robustness, on the other hand, seeks to produce provable guarantees on the adversarial robustness of a model which hold regardless of the attack strategy.
However, the study of provable robustness has mostly focused on perturbations with a fixed size budget (e.g., an $\ell_p$-ball of same size) for all input points~\cite{cohen19, LecuyerAG0J19, LiCWC19, SalmanLRZZBY19, gowal2018effectiveness, HuangSWDYGDK19, WongK18, Raghunathan2018, Singla2019, singla2020secondorder, Levine-ICML21, Levine2020patch, LevineF20aaai}.
Among provable robustness methods, randomized smoothing based procedures have been able to successfully scale up to high-dimensional problems \cite{cohen19, LecuyerAG0J19, LiCWC19, SalmanLRZZBY19} and adapted effectively to other domains such as reinforcement learning \cite{PolicySmoothing, wu2021crop} and models with structured outputs \cite{kumar2021center}.
Existing techniques do not extend to the aforementioned data shifts as the perturbations for each input point in the distribution need not have a fixed bound on their size.
For example, stochastic changes in the input images of a vision model caused by lighting and weather conditions may vary across time and location.
Even adversarial attacks may adjust the perturbation size depending on the input instance.

Robustness against natural as well as adversarial shifts in the distribution is a major challenge for deploying machine learning algorithms in safety-critical applications such as self-driving, medical diagnosis and critical infrastructure.
Verifiable robustness to both these shifts is an important goal in designing reliable real-world systems.
A standard way to describe such data shifts is to constrain the Wasserstein distance between the original distribution $\mathcal{D}$ and the shifted distribution $\mathcal{\tilde{D}}$, i.e., $W_1^{d}(\mathcal{D}, \mathcal{\tilde{D}}) \leq \epsilon$, for an appropriate distance function $d$ and a real number $\epsilon$.
It essentially bounds the average perturbation size across the entire distribution instead of putting a hard budget constraint for each input point.
Wasserstein distance is a standard similarity measure for probability distributions and has been extensively used to study distribution shifts \cite{CourtyFHR17, DamodaranKFTC18, LeeR18, WuWKL19}.
Certifiable robustness against Wasserstein shifts is an interesting problem to study in its own right and a useful tool to have in the arsenal of provable robustness techniques in machine learning.

In this work, we present an efficient procedure to make any model robust to distributional shifts with verifiable guarantees on performance.  We consider families of parameterized distribution shifts which may include shifts in the RBG color balance of an image, the hue/saturation balance, the brightness/contrast, and more.
To achieve distributional robustness, we randomize the input to our model by replacing each input image with a shifted image randomly sampled from a ``smoothing'' distribution.
By randomizing in this way, a given model can be made provably robust to any shifted distribution (natural or adversarial) within a Wasserstein radius $\epsilon$ from the original distribution.

We design robustness certificates that bound the difference between the accuracy of the robust model under the distributions $\mathcal{D}$ and $\mathcal{\tilde{D}}$ as a function of $\epsilon$.
For a function $\bar{h}$ representing the performance of the robust model on an input-output pair $(x, y)$, our main theoretical result in Theorem~\ref{thm:dist-robust} shows that
\[\left|\mathbb{E}_{(x_1, y_1) \sim \mathcal{D}} [\bar{h}(x_1, y_1)] - \mathbb{E}_{(x_2, y_2) \sim \mathcal{\tilde{D}}} [\bar{h}(x_2, y_2)]\right| \leq \psi(\epsilon),\]
where $\psi$ is a concave function that bounds the total variation between the smoothing distributions at two input points as a function of the distance between them (condition~(\ref{eq:tv_bnd}) in Section~\ref{sec:notations}).
Such an upper bound always exists for any smoothing distribution because the total variation remains between zero and one as the distance between the two input points increases.
We discuss how to find the appropriate $\psi$ for different smoothing distributions in appendix~\ref{sec:psi-functions}.
The above result places a Lipschitz-like bound on the change in performance with respect to the Wasserstein distance of the shifted distribution.

\begin{figure*}[t]
    \centering
    \includegraphics[width=0.95\textwidth]{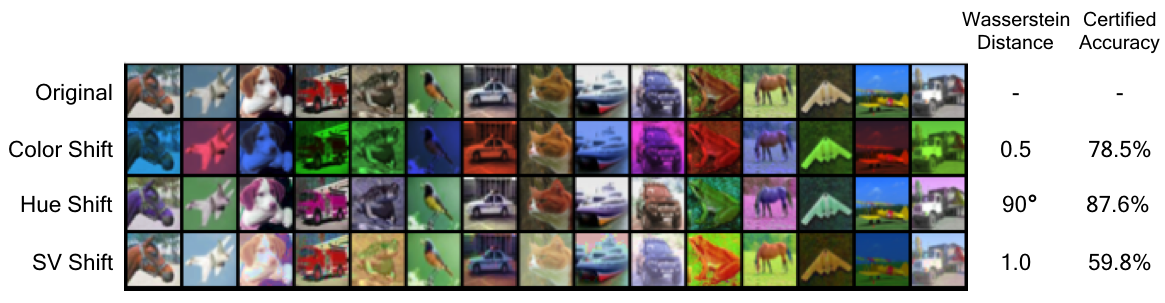}
    \caption{Certified accuracies obtained for different natural transformations of CIFAR-10 images such as color shifts, hue shifts and changes in brightness and saturation. The Wasserstein distance of each distribution shift from the original distribution is defined with respect to the corresponding distance function.}
    \label{fig:transformed_images}
\end{figure*}

Robustness under distribution shifts is a fundamental problem in several areas of machine learning and our certificates could be applicable to a multitude of learning tasks. 
Our method does not make any assumptions on the model or significantly increase its computational requirements as it only needs one sample per input to make robust predictions, making it viable for real-world applications that use conventional neural network architectures.
The sample complexity for generating Wasserstein certificates over the entire distribution is roughly the same as obtaining adversarial certificates for a single input instance using existing randomized smoothing based techniques \cite{cohen19, SalmanLRZZBY19}.
We demonstrate the usefulness of our main theoretical result (Theorem~\ref{thm:dist-robust}) in three domains:

{\bf (i) Certifying model accuracy under natural shifts of the data distribution (Section~\ref{sec:nat-transform}):} We consider three image transformations: color shift, hue shift and changes in brightness and saturation (SV shift).
We define a parameter space for each of these transformations in such a way that the transformations satisfy properties, such as additive composability (equation~(\ref{eq:add_comp})), required for the certificates. 
We show that by randomizing the input of a model in the parameter space of the transformation, we can derive distribution level robustness certificates which guarantee that the accuracy of the model will remain above a threshold for any shifted distribution within an $\epsilon$-sized Wasserstein ball around the original input distribution.
Figure~(\ref{fig:transformed_images}) visualizes CIFAR-10~\cite{krizhevsky2014cifar} images under each of these transformations and reports the corresponding certified accuracies obtained by our method.
Figure~(\ref{fig:undefVScert}) plots the accuracy of two base models (trained on CIFAR-10 images with and without noise in the transformation space) under a shifted distribution and compares it with the certified accuracy of a robust model (noise-trained model smoothed using input randomization).
Note that we commpute the empirical accuracies of the base model (dashed lines) under random, non-adversarial perturbations in the parameter space because naturally occurring image transformations are not adversarial in nature.
An adversarial attack that seeks to find the worst-case case perturbation could further lower the base models' accuracies. 
These results demonstrate that our certificates are significant and non-vacuous (see appendix~\ref{sec:gen-nat-dist} for more details). 

{\bf (ii) Certifying population level robustness against adversarial attacks (Section~\ref{sec:adv-attacks}):}
    The distribution of instances generated by an adversarial attack can also be viewed as a shift in the input distribution within a Wasserstein bound.
    The Wasserstein distance of such a shift is given by the average size of the perturbation ($\ell_2$-norm) added per input instance in the data distribution.
    Since our certificates work for any distribution shift (natural or adversarial) satisfying the Wasserstein bound, we can certify the accuracy of models against adversarial attacks as well.
    Unlike existing certification techniques which assume a fixed perturbation budget across all inputs \cite{cohen19, LecuyerAG0J19, LiCWC19, SalmanLRZZBY19}, our guarantees work for a more general threat model where the adversary is allowed to choose the perturbation size for each input instance as long as it respects the constraint on the average perturbation size over the entire data distribution.
    Also, our procedure only requires {\it one} sample from the smoothing distribution per input instance which makes computing population level certificates significantly more efficient than existing techniques. 
    The certified accuracy we obtain significantly outperforms the base model under attack (figure~\ref{fig:Image_space}).

\begin{figure*}
    \centering
    \includegraphics[width=0.325\textwidth, trim={5mm 2mm 14mm 6mm}, clip]{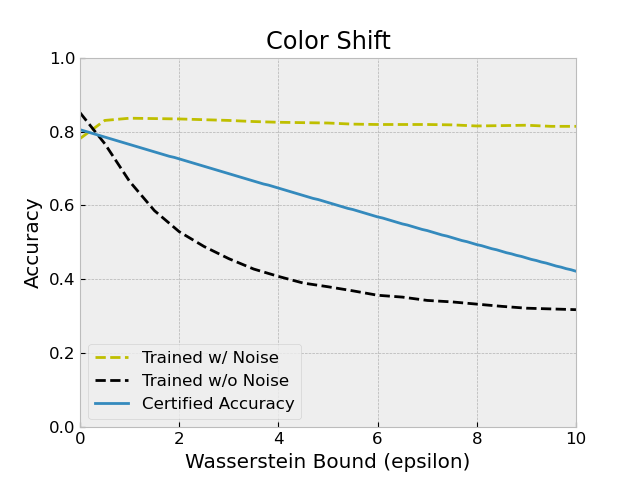}
    \includegraphics[width=0.325\textwidth, trim={5mm 2mm 14mm 6mm}, clip]{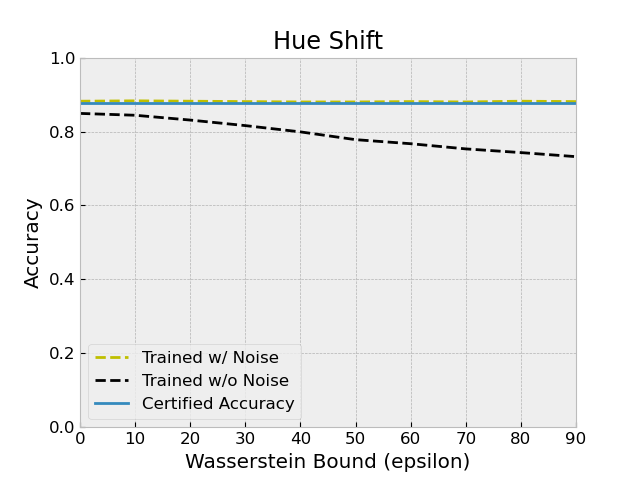}
    \includegraphics[width=0.325\textwidth, trim={5mm 2mm 14mm 6mm}, clip]{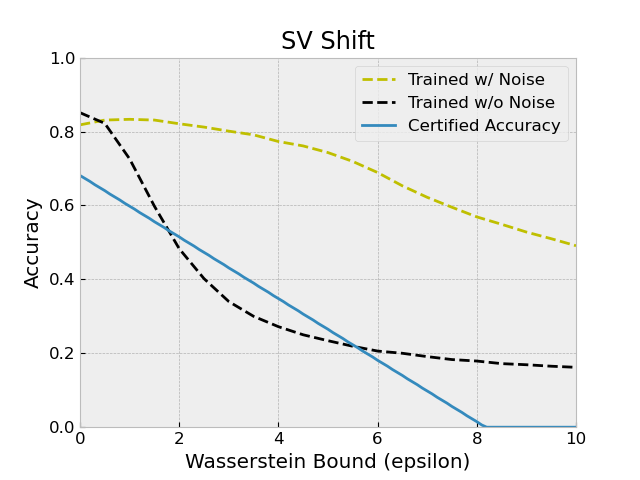}
    \caption{Comparison between the empirical performance (dashed lines) of two base models (trained on CIFAR-10 images with and without noise in transformation space) and the certified accuracy (solid line) of a robust model (noise-trained model smoothed using input randomization) under distribution shifts.
    The certified accuracy 
    often outperforms the undefended model and remains reasonably close (almost overlaps for hue shift) to the model trained under noise for small shifts in the distribution.}
    \label{fig:undefVScert}
\end{figure*}

{\bf (iii) Hardness results for generating “unlearnable” datasets (Section~\ref{sec:unlearnability}):}
    \citep{huang2021unlearnable} proposed a method to make regular datasets unusable for modern deep learning models by poisoning them with adversarial perturbations to interfere with the training of the model.
    The intended purpose is to increase privacy for sensitive data such as personal images uploaded to social media sites.
    The dataset is poisoned in such a way that a model that minimizes the loss on this data distribution will have low accuracy on clean test samples.
    We show that our framework can obtain verifiable lower bounds on the performance of a model trained on such unlearnable datasets. 
    Our certificates guarantee that the performance of the robust model (using input randomization) will remain above a certain threshold on the test distribution even when the base model is trained on the poisoned dataset with a smoothing noise of suitable magnitude. This demonstrates a fundamental limitation in producing unlearnable datasets.

\section{Related Work}
Augmenting training data with corruptions has been shown to improve empirical robustness of machine learning models \cite{HendrycksD19, YangWH19, GoodfellowSS14, MadryMSTV18}.
Training a model with randomly transformed inputs, such as blurring, cropping and rotating images, can improve its test performance against such perturbations.
However, these methods do not seek to produce guarantees on the performance of the model with respect to the extent of the shift in the data distribution.
Our method applies transformations during inference to produce verifiable performance guarantees against any shift in the data distribution within the certified Wasserstein radius.
It does not depend on the underlying model's architecture or the training procedure and may be coupled with robust training techniques to improve the certified guarantees.

Smoothing based approaches that aggregate the predictions on a large number of noisy samples of the input \cite{cohen19, LecuyerAG0J19, LiCWC19, SalmanLRZZBY19} and that use input randomization \cite{Pinot2021} have been studied in the context of certified adversarial robustness.
Certified robustness for parameterized transformations on images also exist \cite{FischerBV20}.
These techniques produce instance-wise certificates and do not generate guarantees on shifts in the data distribution with varying perturbation budget across different points in the distribution.
Our work also differs from instance-wise adversarial attacks and defenses \cite{WongSK19, levine2019wasserstein} that use the Wasserstein distance (instead of conventional $\ell_p$ distances) to measure difference between an image and its perturbed version. In contrast, our certificates consider the Wasserstein distance between data distributions from which images themselves are sampled.

Robustness bounds on the population loss against Wasserstein shifts under the $\ell_2$-distance~\cite{ShenQZY18, SinhaND18} have been derived assuming Lipschitz-continuity of the base model.
These bounds require the Lipschitz constant for the neural network model, which becomes extremely large for deep networks.
Our certificates produce guarantees on the accuracy of an arbitrary neural model without requiring any restrictive assumptions or a global Lipschitz bound.
Additionally, our approach can certify robustness under non-$\ell_p$ changes to the input images, such as visible color shifts, for which the $\ell_2$-norm of the perturbation in the image space will be very large.
Another line of work seeks to prove generalization bounds with respect to divergence-based measures of distribution shift~\cite{Ben-DavidBCP06, Zhao2019, MehraKCH21}.
Divergence measures become arbitrarily large (e.g. KL-divergence becomes infinity) or attain their maximal value (e.g. total variation becomes one) when the supports of probability distributions do not coincide.
Thus, these measures are not suitable for measuring out-of-distribution data shifts which by definition have non-overlapping support.
Wasserstein distance, on the other hand, uses the spatial separation between two distributions to meaningfully measure their distance even when their supports are disjoint.

\section{Preliminaries and Notations}
\label{sec:notations}
Let $\mathcal{D}$ be the data distribution representing a machine learning task over an input space $\mathcal{X}$ and an output space $\mathcal{Y}$.
We define a distribution shift as a covariate shift that only changes the distribution of the input element in samples $(x, y) \in \mathcal{X} \times \mathcal{Y}$ drawn from $\mathcal{D}$ and leaves the output element unchanged, i.e., $(x, y)$ changes to $(\tilde{x}, y)$ under the shift.
Given a distance function $d_{\mathcal{X}} : \mathcal{X} \times \mathcal{X} \rightarrow \mathbb{R}_{\geq 0}$ over the input space, we define the following distance function between two tuples $\tau_1 = (x_1, y_1)$ and $\tau_2 = (x_2, y_2)$ to capture the above shift:%
\begin{equation}
\label{def:shift-dist}
    d (\tau_1, \tau_2) =
    \begin{cases}
      d_{\mathcal{X}}(x_1, x_2) & \text{if } y_1 = y_2\\
      \infty        & \text{otherwise.}
    \end{cases}
\end{equation}
Let $\mathcal{\tilde{D}}$ denote a shift in the original data distribution $\mathcal{D}$ such that the Wasserstein distance under $d$ between $\mathcal{D}$ and $\mathcal{\tilde{D}}$ is bounded by $\epsilon$ (i.e., $W_1^{d}(\mathcal{D}, \mathcal{\tilde{D}}) \leq \epsilon$).
Define the set of all joint probability distributions with marginals $\mu_{\mathcal{D}}$ and $\mu_{\tilde{\mathcal{D}}}$ as follows:%
\begin{align*}
\Gamma(\mathcal{D}, \mathcal{\tilde{D}}) = \bigg\{ \gamma \; \text{ s.t. } \int_{\mathcal{X} \times \mathcal{Y}} \gamma(\tau_1, \tau_2) d\tau_2 = \mu_{\mathcal{D}}(\tau_1)
    \text{ and } \int_{\mathcal{X} \times \mathcal{Y}} \gamma(\tau_1, \tau_2) d\tau_1 = \mu_{\mathcal{\tilde{D}}}(\tau_2) \bigg\}.
\end{align*}
The Wasserstein bound implies that there exists an element $\gamma^* \in \Gamma(\mathcal{D}, \mathcal{\tilde{D}})$ such that%
\begin{equation}
\label{eq:wasserstein_bnd}
\mathbb{E}_{(\tau_1, \tau_2) \sim \gamma^*}[d(\tau_1, \tau_2)] \leq \epsilon.
\end{equation}

Let $\mathcal{S}: \mathcal{X} \rightarrow \Delta(\mathcal{X})$ be a function mapping each element $x \in \mathcal{X}$ to a smoothing distribution $\mathcal{S}(x)$, where $\Delta(\mathcal{X})$ is the set of all probability distributions over $\mathcal{X}$.
For example, smoothing with an isometric Gaussian noise distribution with variance $\sigma^2$ can be denoted as $\mathcal{S}(x) = \mathcal{N}(x, \sigma^2 I)$.
Let the total variation between the smoothing distributions at two points $x_1$ and $x_2$ be bounded by a concave increasing function $\psi$ of the distance between them, i.e.,
\begin{equation}
\label{eq:tv_bnd}
    \TV (\mathcal{S}(x_1), \mathcal{S}(x_2)) \leq \psi(d_{\mathcal{X}}(x_1, x_2)).
\end{equation}
For example, when the distance function $d$ is the $\ell_2$-norm of the difference of $x_1$ and $x_2$, and the smoothing distribution is an isometric Gaussian $\mathcal{N}(0, \sigma^2 I)$ with variance $\sigma^2$, $\psi(\cdot) = \erf(\cdot/2\sqrt{2}\sigma)$ is a valid upper bound on the above total variation that is concave in the positive domain (see appendix~\ref{sec:psi-functions} for more examples).

Consider a function $h: \mathcal{X} \times \mathcal{Y} \rightarrow [0, 1]$ that represents the accuracy of a model.
For example, in the case of a classifier $\mu: \mathcal{X} \rightarrow \mathcal{Y}$ that maps inputs from space $\mathcal{X}$ to a class label in $\mathcal{Y}$, $h(x, y) \coloneqq \mathbf{1}\{\mu(x) = y\}$ could indicate whether the prediction of $\mu$ on $x$ matches the desired output label $y$ or not.
Another example could be that of segmentation/detection tasks, where $y$ represents a region on an input image $x$.
Then, $h(x, y) \coloneqq \text{IoU}(\mu(x), y)$\footnote{IoU stands for Intersection over Union.} could represent the overlap between the predicted regions $\mu(x)$ and the ground truth $y$.
The overall accuracy of the model $\mu$ under $\mathcal{D}$ is then given by $\mathbb{E}_{(x, y) \in \mathcal{D}}[h(x, y)]$.
We define a robust model $\bar{\mu}(x) = \mu(x')$ where $x \sim \mathcal{S}(x)$.
Our goal is to bound the difference in the expected performance of the robust model between the original distribution $\mathcal{D}$ and the shifted distribution $\mathcal{\tilde{D}}$.
We define a {\it smoothed} version of the function $h$ to represent the performance of $\bar{\mu}$ on a given input-output pair $(x, y)$:
\begin{equation}
\label{eq:smooth-fn}
\bar{h}(x, y) = \mathbb{E}_{x' \sim \mathcal{S}(x)} [h(x', y)].
\end{equation}
Our main theoretical result (Theorem~\ref{thm:dist-robust}) bounds the difference between the expected value of $\bar{h}$ under $\mathcal{D}$ and $\mathcal{\tilde{D}}$ using the function $\psi$ and the Wasserstein bound $\epsilon$.

\subsection{Parameterized Transformations}
\label{sec:param_transform}
We apply our distributional robustness certificates to certify the accuracy of an image classifier under natural transformations of the images such as color shifts, hue shifts and changes in brightness and saturation. 
We model each of them as a parameterized transformation $\mathcal{T}: \mathcal{X} \times P \rightarrow \mathcal{X}$ which is a function over the input/image space $\mathcal{X}$ and a parameter space $P$ which takes an image $x \in \mathcal{X}$ and a parameter vector $\theta \in P$ and outputs a transformed image $x' = \mathcal{T}(x, \theta) \in \mathcal{X}$.
An example of such a transformation could be a color shift in an RGB image produced by scaling the intensities in the red, green and blue channels $x = (\{x_{ij}^R\}, \{x_{ij}^G\}, \{x_{ij}^B\})$ defined as $\mathsf{CS}(x, \theta) = (2^{\theta_R} \{x_{ij}^R\}, 2^{\theta_G}\{x_{ij}^G\}, 2^{\theta_B}\{x_{ij}^B\})/\mathsf{MAX}$ for a tuple $\theta = (\theta_R, \theta_G, \theta_B)$, where $\mathsf{MAX}$ is the maximum of all the RGB values after scaling.
Another example could be vector translations $\mathsf{VT}(x, \theta) = x + \theta$.
In order to apply randomized smoothing, we assume that the transformation returns $x$ if the parameters are all zero, i.e., $\mathcal{T}(x, 0) = x$ and that the composition of two transformations with parameters $\theta_1$ and $\theta_2$ is a transformation with parameters $\theta_1 + \theta_2$ (additive composability), i.e.,
\begin{equation}
\label{eq:add_comp}
\mathcal{T}(\mathcal{T}(x, \theta_1), \theta_2) = \mathcal{T}(x, \theta_1 + \theta_2).
\end{equation}

Given a norm $\|\cdot\|$ in the parameter space $P$, we define a distance function in the input space $\mathcal{X}$ as follows:
\begin{equation}
\label{def:trans-dist}
    d_{\mathcal{T}} (x_1, x_2) =
    \begin{cases}
      \min \{\|\theta\| \mid \mathcal{T}(x_1, \theta) = x_2\} & \text{if } \exists \theta \text{ s.t. } \mathcal{T}(x_1, \theta) = x_2\\
      \infty        & \text{otherwise.}
    \end{cases}
\end{equation}

Now, define a smoothing distribution $\mathcal{S}(x) = \mathcal{T}(x, \mathcal{Q}(0))$ for some distribution $\mathcal{Q}$ in the parameter space of $\mathcal{T}$ such that $\forall \theta \in P, \mathcal{Q}(\theta) = \theta + \mathcal{Q}(0)$ is the distribution of $\theta + \delta$ where $\delta \sim \mathcal{Q}(0)$, and $\TV (\mathcal{Q}(0), \mathcal{Q}(\theta)) \leq \psi(\|\theta\|)$ for a concave function $\psi$.
For example, $\mathcal{Q}(\cdot) = \mathcal{N}(\cdot, \sigma^2 I)$ satisfies these properties for $\psi(\cdot) = \erf(\cdot/2\sqrt{2}\sigma)$.
Then, the following lemma holds (proof in appendix~\ref{sec:proof_lem_tv_bnd_transform_dist}):
\begin{lemma}
\label{lem:tv_bnd_transform_dist}
For two points $x_1, x_2 \in \mathcal{X}$ such that $d_{\mathcal{T}}(x_1, x_2)$ is finite, 
\[\TV (\mathcal{S}(x_1), \mathcal{S}(x_2)) \leq \psi(d_{\mathcal{T}}(x_1, x_2)). \]
\end{lemma}

\section{Certified Distributional Robustness}
In this section, we state our main theoretical result showing that the difference in the expectation of the smoothed performance function $\bar{h}$ defined in equation~(\ref{eq:smooth-fn}) under the original distribution $\mathcal{D}$ and the shifted distribution $\mathcal{\tilde{D}}$ can be bounded by the Wasserstein distance between $\mathcal{D}$ and $\mathcal{\tilde{D}}$.
Given a concave upper bound $\psi$ on the total variation between the smoothing distributions at two points $x_1$ and $x_2$ (condition~(\ref{eq:tv_bnd})), the following theorem holds.

\begin{theorem}
\label{thm:dist-robust}
Given a function $h: \mathcal{X} \times \mathcal{Y} \rightarrow [0, 1]$, define its smoothed version as $\bar{h}(x, y) = \mathbb{E}_{x' \sim \mathcal{S}(x)} [h(x', y)]$.
Then,
\[\left|\mathbb{E}_{(x_1, y_1) \sim \mathcal{D}} [\bar{h}(x_1, y_1)] - \mathbb{E}_{(x_2, y_2) \sim \mathcal{\tilde{D}}} [\bar{h}(x_2, y_2)]\right| \leq \psi(\epsilon).\]
\end{theorem}

We defer the proof to the appendix.
The intuition behind the above guarantee is that if the overlap between the smoothing distributions between two individual points does not decrease rapidly with the distance between them, then the overlap between $\mathcal{D}$ and $\mathcal{\tilde{D}}$ augmented with the smoothing distribution remains high when the Wasserstein distance between them is small.
Thus, for a Wasserstein shift of $\epsilon$, the accuracy of the model can be bounded as $\mathbb{E}_{(x_2, y_2) \sim \mathcal{\tilde{D}}} [\bar{h}(x_2, y_2)] \geq \mathbb{E}_{(x_1, y_1) \sim \mathcal{D}} [\bar{h}(x_1, y_1)] - \psi(\epsilon)$.
But, in practice, we may only estimate $\mathbb{E}_{(x_1, y_1) \sim \mathcal{D}} [\bar{h}(x_1, y_1)]$ using a finite number of samples.
In our experiments, we compute a confidence lower bound using the Clopper-Pearson method that holds with $1-\alpha$ probability (for some $\alpha > 0$, usually 0.001) on the robust model's accuracy on the original distribution \cite{clopper_conf_int}.
In the following sections, we will demonstrate the usefulness of this fundamental result in several applications.

\section{Certified Accuracy against Natural Transformations}
\label{sec:nat-transform}
We certify the accuracy of a ResNet-110 model trained on CIFAR-10 images under three types of image transformations: color shifts, hue shifts and variation in brightness and saturation (SV shift).
We train our models with varying levels of noise in the transformation space and evaluate their certified performance using smoothing distributions of different noise levels.
For color and SV shifts, we show how the certified accuracy varies as a function of the Wasserstein distance as we change the training and smoothing noise.
For hue shift, we use a smoothing distribution (with fixed noise level) that is invariant to rotations in hue space and the certified accuracy remains constant with respect to the corresponding Wasserstein distance.
Training each model for 90 epochs takes a few hours on a single NVIDIA GeForce RTX 2080 Ti GPU and computing the distribution level Wasserstein certificates using $10^5$ samples with 99.9\% confidence takes about 25 seconds. 

\begin{figure*}
    \centering
    \hspace{-4mm}
    \includegraphics[width=0.245\textwidth, trim={5mm 2mm 16mm 6mm}, clip]{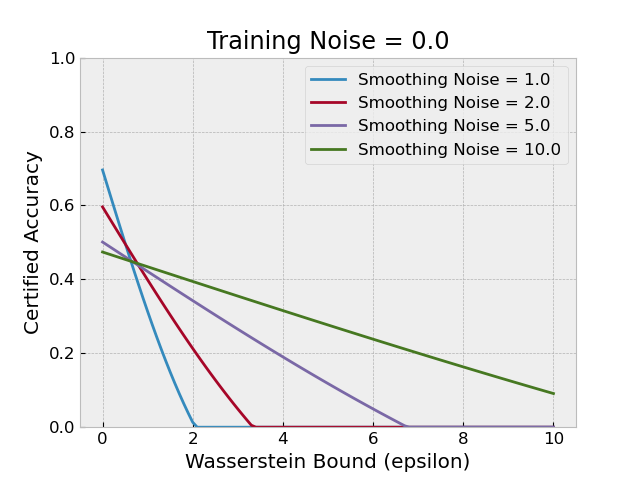}
    \includegraphics[width=0.245\textwidth, trim={5mm 2mm 16mm 6mm}, clip]{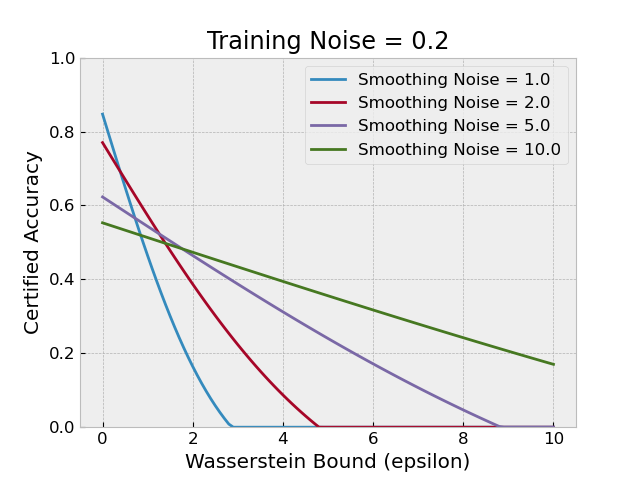}
    \includegraphics[width=0.245\textwidth, trim={5mm 2mm 16mm 6mm}, clip]{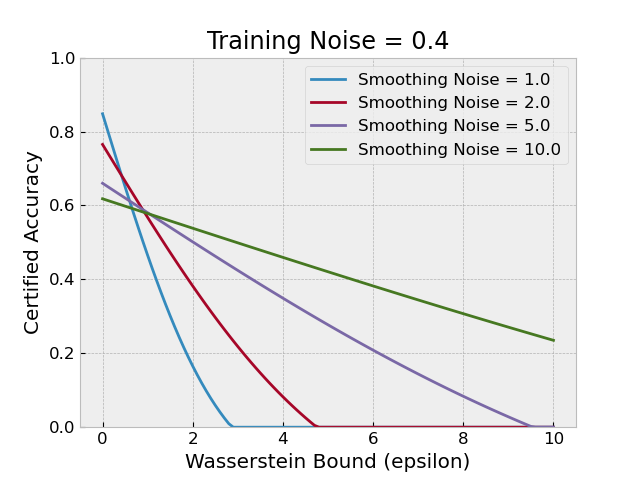}
    \includegraphics[width=0.245\textwidth, trim={5mm 2mm 16mm 6mm}, clip]{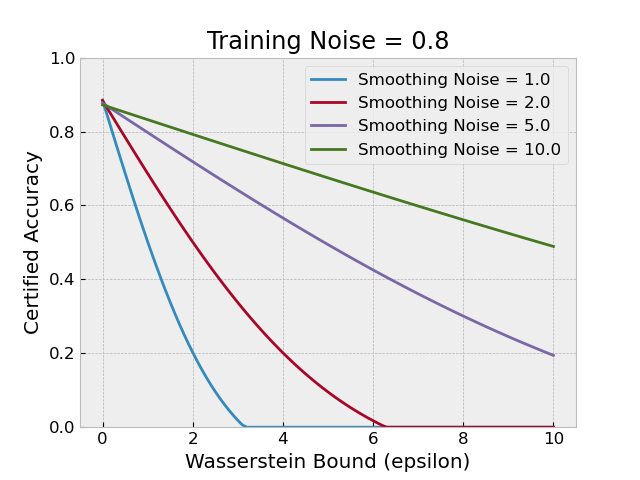}
    \caption{Certified accuracy under color shifts. Each plot corresponds to a particular training noise and each curve corresponds to a particular smoothing noise.}
    \label{fig:cert-acc-CS}
\end{figure*}

\subsection{Color Shifts}
\label{sec:color-shift}

Denote an RGB image $x$ as an $H \times W$ array of pixels where the red, green and blue components of the pixel in the $i$th row and $j$th column are given by the tuple $x_{ij} = (r, g, b)_{ij}$.
Let $r_{\max}, g_{\max}$ and $b_{\max}$ be the maximum values of the red, green and blue channels, respectively.
Assume that the RGB values are in the interval $[0, 1]$ normalized such that the maximum over all intensity values is one, i.e., $\max(r_{\max}, g_{\max}, b_{\max}) = 1$.
Define a color shift of the image $x$ for a parameter vector $\theta \in \mathbb{R}^3$ as
\[\mathsf{CS}(x, \theta) = \left\{ \frac{(2^{\theta_R} r, 2^{\theta_G} g, 2^{\theta_B} b)_{ij} }{\max (2^{\theta_R} r_{\max}, 2^{\theta_G} g_{\max}, 2^{\theta_B} b_{\max})} \right\}^{H \times W}\]
which scales the intensities of each channel by the corresponding component of $\theta$ raised to the power of two and then normalizes the scaled image so that the maximum intensity is one.
For example, $\theta = (1, -1, 0)$ would first double all the red intensities, halve the green intensities and leave the blue intensities unchanged, and then, normalize the image so that the maximum intensity value over all the channels is equal to one.
The above transformation can be shown to satisfy the additive composability property in condition~(\ref{eq:add_comp}). See appendix~\ref{sec:add-comp-proof} for a proof.


Given an image $x$, we define a smoothing distribution around $x$ in the parameter space as $\mathsf{CS}(x, \delta)$ where $\delta \sim \mathcal{N}(0, \sigma^2 I_{3 \times 3})$.
Define the distance function $d_{\mathsf{CS}}$ as described in~(\ref{def:trans-dist}) using the $\ell_2$-norm in the parameter space.
For a distribution $\mathcal{\tilde{D}}$ within a Wasserstein distance of $\epsilon$ 
from the original distribution $\mathcal{D}$, the performance of the smoothed model on $\mathcal{\tilde{D}}$ can be bounded as $\mathbb{E}_{(x_2, y_2) \sim \mathcal{\tilde{D}}} [\bar{h}(x_2, y_2)] \geq \mathbb{E}_{(x_1, y_1) \sim \mathcal{D}} [\bar{h}(x_1, y_1)] - \erf (\epsilon/2 \sqrt{2} \sigma)$.
Figure~\ref{fig:cert-acc-CS} plots the certified accuracy under color shift with respect to the Wasserstein bound $\epsilon$ for different values of training and smoothing noise.
In appendix~\ref{sec:random-channel}, we consider a smoothing distribution that randomly picks one color channel achieving a constant certified accuracy of 87.1\% with respect to the Wasserstein bound $\epsilon$.

\subsection{Hue Shift}
Any RGB image can be alternatively represented in the HSV image format by mapping the $(r, g, b)$ tuple for each pixel to a point $(h, s, v)$ in a cylindrical coordinate system where the values $h, s$ and $v$ represent the hue, saturation and brightness (value) of the pixel.
The mapping from the RGB coordinate to the HSV coordinate takes the $[0, 1]^3$ color cube and transforms it into a cylinder of unit radius and height.
The hue values are represented as angles in $[0, 2\pi)$ and the saturation and brightness values are in $[0, 1]$.
Define a hue shift of an $H \times W$ sized image $x$ by an angle $\theta \in [-\pi, \pi]$
in the HSV space
that rotates each hue value by an angle $\theta$ and wraps it around to the $[0, 2\pi)$ range.
In appendix~\ref{sec:apx-hue-plots}, we show that the certified accuracy under hue shifts does not depend on the Wasserstein distance of the shifted distribution and report the certified accuracies obtained by various base models trained under different noise levels

\begin{figure*}
    \centering
    \hspace{-4mm}
    \includegraphics[width=0.245\textwidth, trim={5mm 2mm 16mm 6mm}, clip]{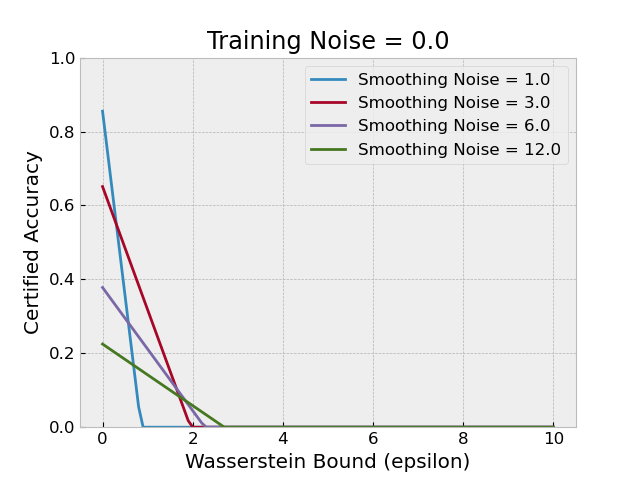}
    \includegraphics[width=0.245\textwidth, trim={5mm 2mm 16mm 6mm}, clip]{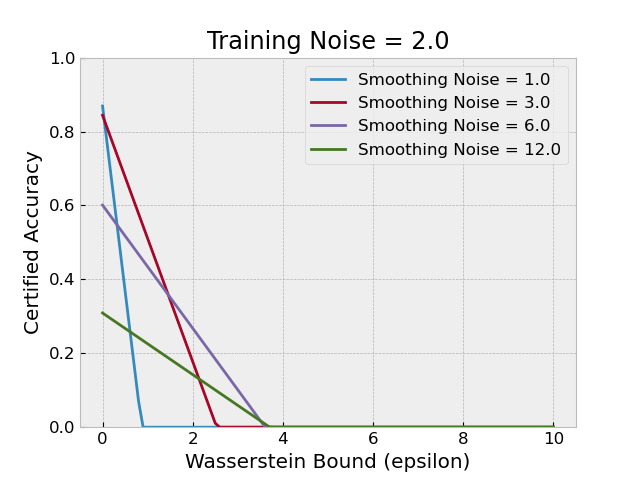}
    \includegraphics[width=0.245\textwidth, trim={5mm 2mm 16mm 6mm}, clip]{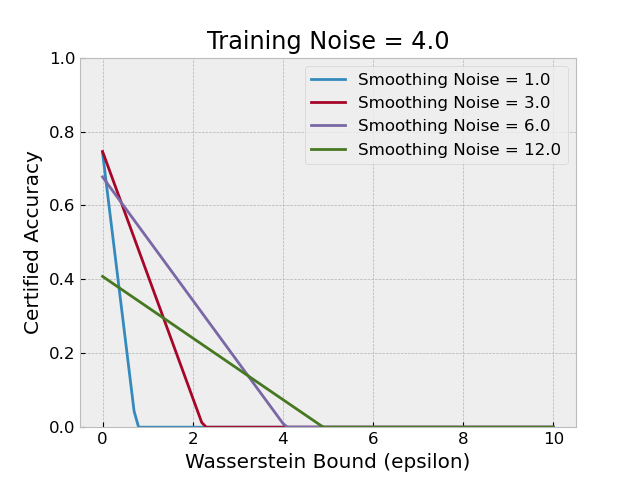}
    \includegraphics[width=0.245\textwidth, trim={5mm 2mm 16mm 6mm}, clip]{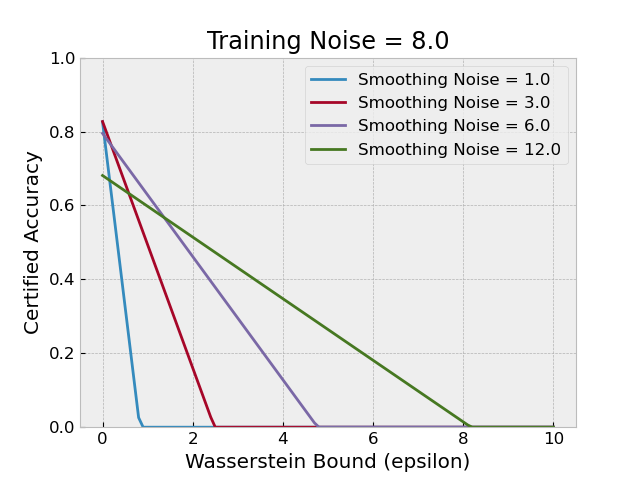}
    \caption{Certified accuracy under brightness and saturation changes. Each plot corresponds to a particular training noise and each curve corresponds to a particular smoothing noise.}
    \label{fig:cert-acc-SV}
\end{figure*}

\subsection{Brightness and Saturation Changes}
\label{sec:brightness-saturation}
Define the following transformation in the HSV space of an image that shifts the mean of the saturation (S) and brightness (V) values for each pixel by a certain amount:
\begin{align*}
\mathsf{SV}(x, \theta) = \Bigg\{ \bigg( h, &\frac{s + (2^{\theta_S} - 1) s_{\mean}}{\mathsf{MAX}},
\frac{v + (2^{\theta_V} - 1) v_{\mean}}{\mathsf{MAX}} \bigg)_{ij} \Bigg\}^{H \times W}
\end{align*}
where $s_{\mean}, s_{\max}, v_{\mean}$ and $v_{\max}$ are the means and maximums of the saturation and brightness values respectively before the shift is applied and $\mathsf{MAX} = \max(s_{\max} + (2^{\theta_S} - 1) s_{\mean}, v_{\max} + (2^{\theta_V} - 1) v_{\mean})$ is the maximum of the brightness and saturation values after the shift.
Similar to color shift , the $\mathsf{SV}$ transformation can also be shown to satisfy additive composability (appendix~\ref{sec:add-comp-proof}).

Figure~\ref{fig:cert-acc-SV} plots the certified accuracy under saturation and brightness changes with respect to the Wasserstein bound $\epsilon$ for different values of training and smoothing noise.
The smoothing distribution for this transformation is a uniform in the range $[0, a]^2$ in the parameter space, the distance function is the $\ell_1$-norm and $\psi (\epsilon) = \min(\epsilon / a, 1)$.

\section{Population-Level Certificates against Adversarial Attacks}
\label{sec:adv-attacks}
In this section, we consider the $\ell_2$-distance in the image space to measure the Wasserstein distance instead of a parameterized transformation.
We use a pixel-space Gaussian smoothing distribution $\mathcal{S}(x) = \mathcal{N}(x, \sigma^2 I)$ to obtain robustness guarantees under this metric.
To motivate this, consider an adversarial attacker $\text{Adv}: \mathcal{X} \to \mathcal{X} $, which takes an image $x$ and computes perturbation $\text{Adv}(x)$ to try and fool a model into misclassifying the input.
If $(x, y) \sim \mathcal{D}$, define $\mathcal{\tilde{D}}$ to be the distribution of the tuples $(\text{Adv}(x), y)$.
Defining $d$ in~\ref{def:shift-dist} using $d_{\mathcal{X}} = \ell_2$, it is easy to show that:
\begin{equation} \label{eq:wass_exp}
    W_1^{d}(\mathcal{D}, \mathcal{\tilde{D}})  \leq \mathbb{E}_{x \sim \mathcal{D}}  [\|\text{Adv}(x) -  x\|_2 ]
\end{equation}
So, if the \textit{average} magnitude of perturbations induced by $\text{Adv}$ is less than $\epsilon$ (i.e., $ [\|\text{Adv}(x) -  x\|_2 ] <  \epsilon$), then $W_1^{d}(\mathcal{D}, \mathcal{\tilde{D}})  < \epsilon$  
which means that we can apply Theorem \ref{thm:dist-robust}: the gap in the expected accuracy between $x \sim \mathcal{D}$ and $\text{Adv}(x)  \sim \mathcal{\tilde{D}} $ will be at most $\psi(\epsilon)$.
Note that, under this threat model, $\text{Adv}$ can be strategic in its use of the average perturbation ``budget'': if a certain point $x$ would require a very large perturbation to be misclassified, or is already misclassified, then $\text{Adv}(x)$ can save the budget by simply returning $x$ and use it to attack a greater number of more vulnerable samples.  

Note that our method differs from \textit{sample-wise} certificates against $\ell_2$ adversarial attacks which use randomized smoothing, such as \cite{cohen19}. Specifically, we use only one smoothing perturbation (and therefore only one forward pass) per sample. Our guarantees are on the overall accuracy of the classifier, not on the stability of any particular prediction. Finally, as discussed, our threat model is different, because we allow the adversary to strategically choose which samples to attack, with the certificate dependent on the \textit{Wasserstein} magnitude of the \textit{distributional} attack.

\begin{figure}[t]
\centering
\begin{subfigure}{.49\textwidth}
    \vspace{-2mm}
    \centering
    \includegraphics[width=\linewidth]{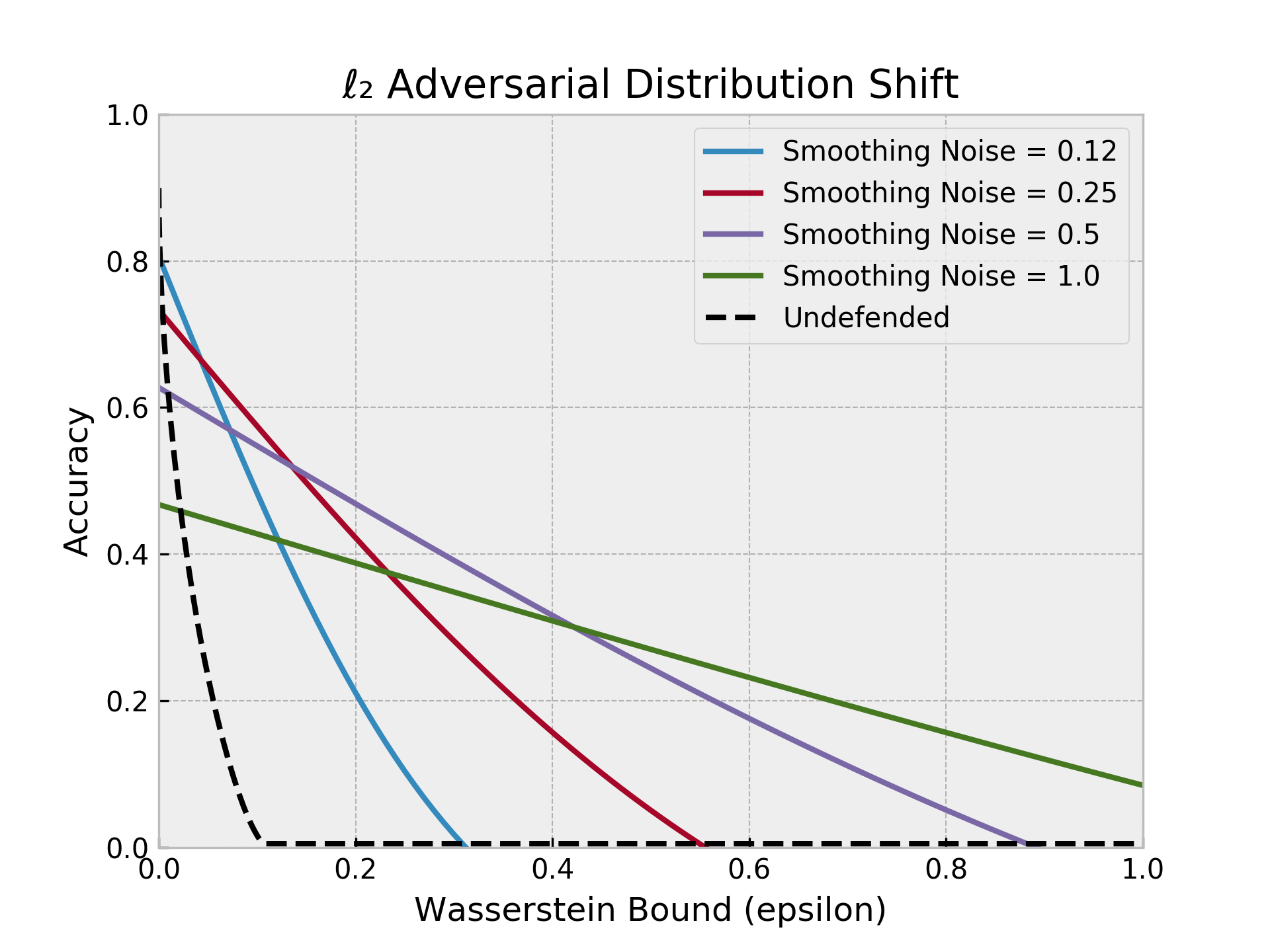}
    \caption{$\ell_2$ adversarial attacks.}
    \label{fig:Image_space}
\end{subfigure}
\begin{subfigure}{.49\textwidth}
  \centering
    \includegraphics[width=0.96\linewidth]{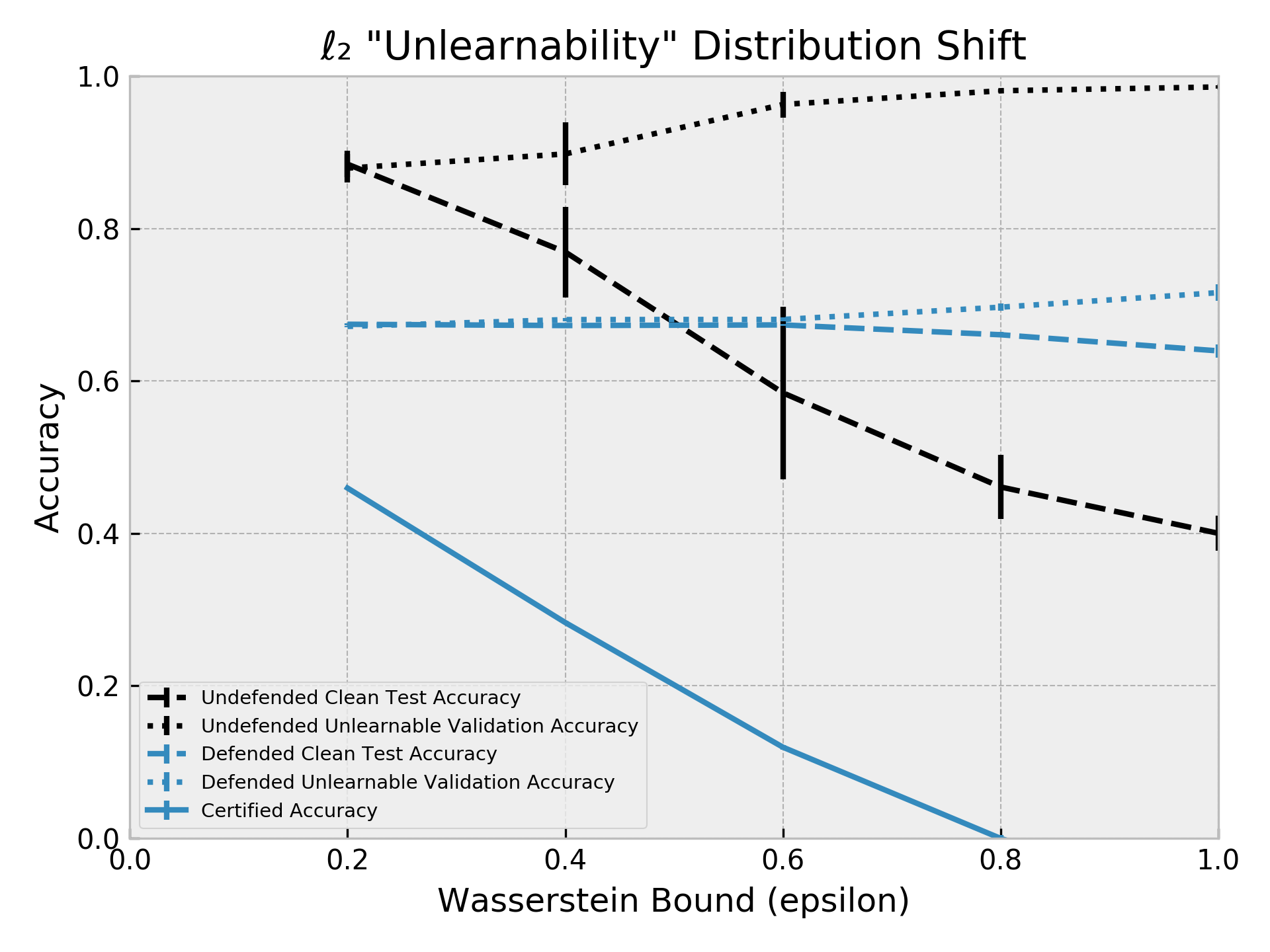}
    \caption{Unlearnability.}
    \label{fig:unlearnablity}
\end{subfigure}
\caption{Distributional certificates against (a) adversarial attacks and (b) unlearnable datasets. The smoothing noise used in part (b) is $0.4$. Results for other values are reported in the appendix.}

\end{figure}

Results on CIFAR-10 are presented in Figure \ref{fig:Image_space}. We use ResNet-110 models trained under noise from \cite{cohen19}. For the undefended baseline, (on an undefended classifier $g$), we first apply a Carlini and Wagner $\ell_2$ attack to each sample $x$ \citep{Carlini017}, generating adversarial examples $x'$. Define this attack as the function $CW(\cdot)$, such that $x' = CW(x, y; g)$, where $y$ is the ground-truth label. (If the attack fails, $CW(x, y; g) = x$). We then define a \textit{strategic} adversary  $\text{Adv}_\gamma$ that returns $CW(x, y; g)$ if $\|CW(x, y; g) - x\|_2 < \gamma $, otherwise it returns $x$.
By not attacking samples which would require the largest $\ell_2$ perturbations to cause misclassification, this attack efficiently balances maximizing misclassification rate with minimizing the Wasserstein distance between $\mathcal{D}$ and $\mathcal{\tilde{D}}$. The threshold parameter $\gamma$ controls the tradeoff between misclassifcation rate and the Wasserstein perturbation magnitude. The `Undefended' baseline in Figure \ref{fig:Image_space} plots the accuracy on attacked test samples under adversary $\text{Adv}_\gamma$, for a sweep of values of $\gamma$, against an upper bound on the Wasserstein distance, given by  $\mathbb{E}_{x \sim \mathcal{D}_{test}}  [\|\text{Adv}_\gamma(x) -  x\|_2 ]$. (Note that here we are using the finite test distribution $\mathcal{D}_{test}$, so this is technically an upper bound on  $ W_1^{d}(\mathcal{D}_{test}, \mathcal{\tilde{D}}_{test})$). 
We can observe a large gap between this undefended model performance under attack, and the certified robustness of our model, showing that our certificate is highly nonvacuous.
In the supplemental material,
we include results to show the empirical robustness of the smoothed classifiers under an ``adaptive'' attack, based on the attack on sample-wise $\ell_2$ smoothing proposed by \cite{SalmanLRZZBY19}.

\section{Hardness Results on Unlearnability}
\label{sec:unlearnability}

In this section, we show that the pixel-space $\ell_2$-Wasserstein distributional robustness certificate shown above can also be applied to establish a hardness result in creating provably ``unlearnable'' datasets  \cite{huang2021unlearnable}. In these datasets, every sample of the released data is ``poisoned'' so that the accuracy of a classifier trained on this data is high on both the training set and any other (i.e., validation) set split from the released dataset, while the test accuracy on non-poisoned samples drawn from the same distribution is low. This technique has legitimate applications, such as protecting privacy by preventing one's personal data from being learned, but may also have malicious uses (e.g., a malicious actor could sell a useless classifier that nevertheless has good performance on a provided validation set.) We can view the ``clean'' data distribution as $\mathcal{D}$, and the distribution of the poisoned samples (i.e., the unlearnable distribution) as $\mathcal{\tilde{D}}$. If the magnitude of the perturbations is limited, Theorem \ref{thm:dist-robust} implies that the accuracy on $\mathcal{D}$ and $\mathcal{\tilde{D}}$ must be similar, implying that our algorithm is provably resistant to unlearnablility attacks, effectively establishing provable hardness results to create unlearnable datasets. 

In order to apply our guarantees, we must make a few modifications to the attack proposed in \cite{huang2021unlearnable}. First, we bound each poisoning perturbation on the released dataset to within an $\epsilon$-radius $\ell_2$  ball, rather than an $\ell_\infty$ ball. From Equation \ref{eq:wass_exp}, this ensures that $W_1^{d}(\mathcal{D}, \mathcal{\tilde{D}}) \leq \epsilon.$

\looseness -1
Second, we consider an ``offline'' version of the attack. In the original attack \cite{huang2021unlearnable}, perturbations for the entire dataset are optimized simultaneously with a proxy classifier model in an iterative manner. This makes the perturbations applied to each sample non-I.I.D., (because they may depend on each other through proxy-model parameters) 
which makes deriving generalizable guarantees for it difficult. 
However, this simultaneous proxy-model training and poisoning may not always represent a realistic threat model. In particular, an actor releasing  ``unlearnable'' data at scale may not be able to constantly update the proxy model being used. For example, consider an ``unlearnability'' module in a camera, which would make photos unusable as training data. Because the camera itself has access to only a small number of photographs, such a module would likely rely on a fixed, pre-trained proxy classifier model to create the poisoning perturbations. To model this, we consider a threat model where the proxy classifier is first optimized using an unreleased dataset: the released ``unlearnable'' samples are then perturbed independently using this fixed proxy model. We see in Figure \ref{fig:unlearnablity} that our modified attack is still highly effective at making data unlearnable, as shown by the high validation and low test accuracy of the undefended baseline.

In Figure \ref{fig:unlearnablity}, we also show the performance of our algorithm on CIFAR-10 under unlearnability attack. We use an ``adaptive'' attack against the smoothed classifier, inspired by \cite{SalmanLRZZBY19}: details are presented in the appendix. The defense is empirically effective at thwarting unlearnability: the poisoned validation and clean test sets have similar accuracies, even at large perturbation size. However, overall accuracy is reduced substantially. Certified lower bounds on the clean accuracy, computed from the poisoned validation accuracy using Theorem \ref{thm:dist-robust}, are also given. 

\section{Conclusion}
\label{sec:limitations}
In this work, we show that it is possible to certify distributional robustness under natural as well as adversarial shifts in the data distribution.
Our method can make any model provably robust without increasing its sample complexity.
We certify accuracy with respect to the Wasserstein distance of the distribution shift which is a more suitable metric for out-of-distribution shifts than previously considered divergence measures such as KL-divergence and total variation.
The robustness guarantees we derive do not make any assumptions on the base model such as Lipschitz-continuity, making our method suitable for most real-world applications using deep neural network architectures.

We show that by appropriately parameterizing the transformation space, one can obtain meaningful certificates for natural shifts that have a high perturbation size in image space.
However, the distance functions we consider, such as $\ell_2$ and parameterized transformations, are predefined non-learnable functions which may not be suitable for modeling more sophisticated data shifts such as perceptual changes.
A future direction of research could be to adapt our distributional certificates for more difficult domains such as weather patterns, user preferences, facial expressions, etc.
We do not foresee any immediate negative impact of our work on society.
It seeks to make machine learning models more robust and reliable against unexpected data shifts in the real world.

\section{Acknowledgements}
This project was supported in part by NSF CAREER AWARD 1942230, a grant from NIST 60NANB20D134, HR001119S0026 (GARD), the Office of Naval Research (N000142112557), the ONR YIP award N00014-22-1-2271, Army Grant No. W911NF2120076 and NSF award CCF2212458.
Further support was provided by the AFOSR MURI program.

\bibliography{references}

\appendix
\section{Proof of Theorem~\ref{thm:dist-robust}}

\begin{statement}
Given a function $h: \mathcal{X} \times \mathcal{Y} \rightarrow [0, 1]$, define its smoothed version as $\bar{h}(x, y) = \mathbb{E}_{x' \sim \mathcal{S}(x)} [h(x', y)]$.
Then,
\[\left|\mathbb{E}_{(x_1, y_1) \sim \mathcal{D}} [\bar{h}(x_1, y_1)] - \mathbb{E}_{(x_2, y_2) \sim \mathcal{\tilde{D}}} [\bar{h}(x_2, y_2)]\right| \leq \psi(\epsilon).\]
\end{statement}

\begin{proof}
Let $\tau_1 = (x_1, y_1)$ and $\tau_2 = (x_2, y_2)$ denote the input-output tuples sampled from $\mathcal{D}$ and $\mathcal{\tilde{D}}$ respectively.
Then, for the joint distribution $\gamma^* \in \Gamma(\mathcal{D}, \mathcal{\tilde{D}})$ in~(\ref{eq:wasserstein_bnd}), we have
\[ \mathbb{E}_{\tau_1 \sim \mathcal{D}} [\bar{h}(\tau_1)] = \mathbb{E}_{(\tau_1, \tau_2) \sim \gamma^*} [\bar{h}(\tau_1)] \quad \text{and} \quad \mathbb{E}_{\tau_2 \sim \mathcal{\tilde{D}}} [\bar{h}(\tau_2)] = \mathbb{E}_{(\tau_1, \tau_2) \sim \gamma^*} [\bar{h}(\tau_2)].\]
This is because when $(\tau_1, \tau_2)$ is sampled from the joint distribution $\gamma^*$, $\tau_1$ and $\tau_2$ individually have distributions $\mathcal{D}$ and $\mathcal{\tilde{D}}$ respectively.
Also, since the expected distance between $\tau_1 = (x_1, y_1)$ and $\tau_2 = (x_2, y_2)$ is finite, the output elements of the sampled tuples must be the same, i.e. $y_1 = y_2 = y$ (say).
See lemma~\ref{lem:zero-prob} below.
Then,
\begin{align*}
    \big| \mathbb{E}_{(x_1, y_1) \sim \mathcal{D}} [\bar{h}(x_1, y_1)] &- \mathbb{E}_{(x_2, y_2) \sim \mathcal{\tilde{D}}} [\bar{h}(x_2, y_2)] \big|\\
    &= \left| \mathbb{E}_{\tau_1 \sim \mathcal{D}} [\bar{h}(\tau_1)] - \mathbb{E}_{\tau_2 \sim \mathcal{\tilde{D}}} [\bar{h}(\tau_2)] \right|\\
    &= \left|\mathbb{E}_{(\tau_1, \tau_2) \sim \gamma^*} [\bar{h}(\tau_1) - \bar{h}(\tau_2)]\right|\\
    & \leq \mathbb{E}_{(\tau_1, \tau_2) \sim \gamma^*} [|\bar{h}(\tau_1) - \bar{h}(\tau_2)|].
\end{align*}

Now, from definition~(\ref{eq:smooth-fn}) and for $i = 1$ and $2$,
\[\bar{h}(\tau_i) = \bar{h}(x_i, y) = \mathbb{E}_{x_i' \sim \mathcal{S}(x_i)} [h(x_i', y)] = \mathbb{E}_{x_i' \sim \mathcal{S}(x_i)} [g(x_i')]\]
can be expressed as the expected value of a function $g: \mathcal{X} \rightarrow [0, 1]$ under distribution $\mathcal{S}(x_i)$.
Without loss of generality, assume $\mathbb{E}_{x_1' \sim \mathcal{S}(x_1)} [g(x_1')] \geq \mathbb{E}_{x_2' \sim \mathcal{S}(x_2)} [g(x_2')]$. Then,

\begin{align*}
    \big| \mathbb{E}_{x_1' \sim \mathcal{S}(x_1)} [g(x_1')] &- \mathbb{E}_{x_2' \sim \mathcal{S}(x_2)} [g(x_2')] \big|\\
    & = \int_{\mathcal{X}} g(x) \mu_1(x) dx - \int_{\mathcal{X}} g(x) \mu_2(x) dx \tag{$\mu_1$ and $\mu_2$ are the PDFs of $\mathcal{S}(x_1)$ and $\mathcal{S}(x_1)$}\\
    &= \int_{\mathcal{X}} g(x) (\mu_1(x) - \mu_2(x)) dx \\
    &= \int_{\mu_1 > \mu_2} g(x) (\mu_1(x) - \mu_2(x)) dx - \int_{\mu_2 > \mu_1} g(x) (\mu_2(x) - \mu_1(x)) dx\\
    &\leq \int_{\mu_1 > \mu_2} \max_{x' \in \mathcal{X}}g(x') (\mu_1(x) - \mu_2(x)) dx - \int_{\mu_2 > \mu_1} \min_{x' \in \mathcal{X}}g(x') (\mu_2(x) - \mu_1(x)) dx\\
    &\leq \int_{\mu_1 > \mu_2} (\mu_1(x) - \mu_2(x)) dz \tag{since $\max_{x' \in \mathcal{X}} g(x') \leq 1$ and $\min_{x' \in \mathcal{X}} g(x') \geq 0$}\\
    &= \frac{1}{2} \int_{\mathcal{X}} |\mu_1(x) - \mu_2(x)| dx =\TV(\mathcal{S}(x_1), \mathcal{S}(x_2)). \tag{since $\int_{\mu_1 > \mu_2} (\mu_1(x) - \mu_2(x)) dx = \int_{\mu_2 > \mu_1} (\mu_2(x) - \mu_1(x)) dx = \frac{1}{2} \int_{\mathcal{X}} |\mu_1(x) - \mu_2(x)| dx$}\\
\end{align*}
Thus, from~(\ref{def:shift-dist}) and~(\ref{eq:tv_bnd}), we have $|\bar{h}(\tau_1) - \bar{h}(\tau_2)| \leq \psi(d_\mathcal{X}(x_1, x_2)) = \psi(d(\tau_1, \tau_2))$, and therefore,
\begin{align*}
    \big| \mathbb{E}_{(x_1, y_1) \sim \mathcal{D}} [\bar{h}(x_1, y_1)] &- \mathbb{E}_{(x_2, y_2) \sim \mathcal{\tilde{D}}} [\bar{h}(x_2, y_2)] \big|\\
    & \leq \mathbb{E}_{(\tau_1, \tau_2) \sim \gamma^*} [\psi(d(\tau_1, \tau_2))]\\
    & \leq \psi \left( \mathbb{E}_{(\tau_1, \tau_2) \sim \gamma^*} [d(\tau_1, \tau_2)] \right). \tag{$\psi$ is concave, Jensen's inequality}
\end{align*}
Hence, from~(\ref{eq:wasserstein_bnd}) and since $\psi$ is non-decreasing, we have
\[\left| \mathbb{E}_{(x_1, y_1) \sim \mathcal{D}} [\bar{h}(x_1, y_1)] -  \mathbb{E}_{(x_2, y_2) \sim \mathcal{\tilde{D}}} [\bar{h}(x_2, y_2)] \right| \leq \psi(\epsilon).\]
\end{proof}

\begin{lemma}
\label{lem:zero-prob}
Let $\Omega = \{(\tau_1, \tau_2) \text{ s.t. } y_1 \neq y_2 \text{ where } \tau_1 = (x_1, y_1) \text{ and } \tau_2 = (x_2, y_2)\}$. Then
\[\mathbb{P}_{(\tau_1, \tau_2) \sim \gamma^*}[(\tau_1, \tau_2) \in \Omega] = 0.\]
\end{lemma}

\begin{proof}
Assume, for the sake of contradiction, that
\[\mathbb{P}_{(\tau_1, \tau_2) \sim \gamma^*}[(\tau_1, \tau_2) \in \Omega] \geq p\]
for some $p > 0$.
From condition~(\ref{eq:wasserstein_bnd}), we have
\[\mathbb{E}_{(\tau_1, \tau_2) \sim \gamma^*}[d(\tau_1, \tau_2)] \leq \epsilon.\]
By the law of total expectation
\begin{align*}
    \mathbb{E}_{\gamma^*}[d(\tau_1, \tau_2)] = &\mathbb{E}_{\gamma^*}[d(\tau_1, \tau_2) \mid (\tau_1, \tau_2) \in \Omega] \; \mathbb{P}_{\gamma^*}[(\tau_1, \tau_2) \in \Omega]\\
    + &\mathbb{E}_{\gamma^*}[d(\tau_1, \tau_2) \mid (\tau_1, \tau_2) \notin \Omega] \; \mathbb{P}_{\gamma^*}[(\tau_1, \tau_2) \notin \Omega].
\end{align*}
We replace $(\tau_1, \tau_2) \sim \gamma^*$ with just $\gamma^*$ in the subscripts for brevity. Since both summands are non-negative,
\[\mathbb{E}_{\gamma^*}[d(\tau_1, \tau_2) \mid (\tau_1, \tau_2) \in \Omega] \; \mathbb{P}_{\gamma^*}[(\tau_1, \tau_2) \in \Omega] \leq \epsilon.\]
Consider a real number $l > \epsilon / p$.
Then, for any $(\tau_1, \tau_2) \in \Omega$, from definition~(\ref{def:shift-dist}) and because $y_1 \neq y_2$, $d(\tau_1, \tau_2) \geq l$.
Therefore, $\mathbb{E}_{\gamma^*}[d(\tau_1, \tau_2) \mid (\tau_1, \tau_2) \in \Omega] \geq l$ and
\begin{align*}
    l \; \mathbb{P}_{\gamma^*}[(\tau_1, \tau_2) \in \Omega] &\leq \mathbb{E}_{\gamma^*}[d(\tau_1, \tau_2) \mid (\tau_1, \tau_2) \in \Omega] \; \mathbb{P}_{\gamma^*}[(\tau_1, \tau_2) \in \Omega]\\
    l \; \mathbb{P}_{\gamma^*}[(\tau_1, \tau_2) \in \Omega] &\leq \epsilon\\
    \mathbb{P}_{\gamma^*}[(\tau_1, \tau_2) \in \Omega] &\leq \epsilon / l < p,
\end{align*}
which contradicts our initial assumption.
\end{proof}

\section{Proof of Lemma~\ref{lem:tv_bnd_transform_dist}}
\label{sec:proof_lem_tv_bnd_transform_dist}
\begin{statement}
For two points $x_1, x_2 \in \mathcal{X}$ such that $d_{\mathcal{T}}(x_1, x_2)$ is finite,
\[\TV (\mathcal{S}(x_1), \mathcal{S}(x_2)) \leq \psi(d_{\mathcal{T}}(x_1, x_2)). \]
\end{statement}

\begin{proof}
Consider the $\theta$ for which $d_{\mathcal{T}}(x_1, x_2) = \|\theta\|$.
Then, $\mathcal{T}(x_1, \theta) = x_2$.
\begin{align*}
    \TV (\mathcal{S}(x), \mathcal{S}(x_2)) &= \TV (\mathcal{T}(x, \mathcal{Q}(0)), \mathcal{T}(z, \mathcal{Q}(0)))\\
    &=\TV (\mathcal{T}(x, \mathcal{Q}(0)), \mathcal{T}(\mathcal{T}(x, \theta), \mathcal{Q}(0)))\\
    &=\TV (\mathcal{T}(x, \mathcal{Q}(0)), \mathcal{T}(x, \theta + \mathcal{Q}(0)))\tag{additive composability, equation~(\ref{eq:add_comp})}\\
    &=\TV (\mathcal{T}(x, \mathcal{Q}(0)), \mathcal{T}(x, \mathcal{Q}(\theta))). \tag{definition of $\mathcal{Q}$}
\end{align*}
Let $A$ be the event in the space $M$ that maximizes the difference in the probabilities assigned to $A$ by $\mathcal{T}(x, \mathcal{Q}(0))$ and $\mathcal{T}(x, \mathcal{Q}(\theta))$.
Let $u: P \rightarrow [0,1]$ be a function that returns the probability (over the randomness of $\mathcal{T}$) of any parameter $\eta \in P$ being mapped to a point in $A$, i.e., $u(\eta) = \mathbb{P}\{\mathcal{T}(x, \eta) \in A\}$.
For a deterministic transformation $\mathcal{T}$, $u$ is a 0/1 function.
Then, the probabilities assigned by $\mathcal{T}(x, \mathcal{Q}(0))$ and $\mathcal{T}(x, \mathcal{Q}(\theta))$ to $A$ is equal to $\mathbb{E}_{\eta \sim \mathcal{Q}(0)}[u(\eta)]$ and $\mathbb{E}_{\eta \sim \mathcal{Q}(\theta)}[u(\eta)]$.
Therefore,
\begin{align*}
    \TV (\mathcal{S}(x), \mathcal{S}(x_2)) &= |\mathbb{E}_{\eta \sim \mathcal{Q}(0)}[u(\eta)] - \mathbb{E}_{\eta \sim \mathcal{Q}(\theta)}[u(\eta)]|\\
    &\leq \TV (\mathcal{Q}(0), \mathcal{Q}(\theta))\\
    & \leq \psi(\|\theta\|) = \psi(d_{\mathcal{T}}(x_1, x_2)). \tag{definition of $\mathcal{Q}$ and $d_{\mathcal{T}}$}
\end{align*}

\end{proof}

\section{Function $\psi$ for Different Distributions}
\label{sec:psi-functions}
For an isometric Gaussian distribution $\mathcal{N}(0, \sigma^2 I)$,
\[\TV (\mathcal{N}(0, \sigma^2 I), \mathcal{N}(\theta, \sigma^2 I)) = \erf (\|\theta\|_2/2 \sqrt{2} \sigma).\]

\begin{proof}
Due to the isometric symmetry of the Gaussian distribution and the $\ell_2$-norm, we may assume, without loss of generality, that $\mathcal{N}(\theta, \sigma^2 I)$ is obtained by shifting $\mathcal{N}(0, \sigma^2 I)$ only along the first dimension.
Therefore, the total variation of the two distributions is equal to the difference in the probability of a normal random variable with variance $\sigma^2$ being less than $\|\theta\|_2/2$ and $-\|\theta\|_2/2$, i.e., $\Phi(\|\theta\|_2/2 \sigma) - \Phi(-\|\theta\|_2/2 \sigma)$ where $\Phi$ is the standard normal CDF.
\begin{align*}
    \TV (\mathcal{N}(0, \sigma^2 I), \mathcal{N}(\theta, \sigma^2 I)) &= \Phi(\|\theta\|_2/2 \sigma) - \Phi(-\|\theta\|_2/2 \sigma)\\
    &= 2\Phi(\|\theta\|_2/2 \sigma) - 1\\
    &= 2 \left( \frac{1 + \erf( \|\theta\|_2 / 2 \sqrt{2} \sigma)}{2} \right) - 1\\
    &= \erf( \|\theta\|_2 / 2 \sqrt{2} \sigma).
\end{align*}
\end{proof}

For a uniform distribution $\mathcal{U}(\theta, b)$ between $\theta_i$ and $\theta_i + b$ in each dimension for $b \geq 0$ (as used for the SV shift transformations), $\TV (\mathcal{U}(0, b), \mathcal{U}(\theta, b)) \leq \|\theta\|_1/b$.
When $\|\theta\|_1$ is constrained, the volume of the overlap between $\mathcal{U}(0, b)$ and $\mathcal{U}(\theta, b)$ is minimized when the shift is only along one dimension.

\section{Additive Composability of Natural Transformations}
\label{sec:add-comp-proof}

In this section, we prove that the natural transformation $\mathsf{CS}, \mathsf{HS}$ and $\mathsf{SV}$ defined in the paper all satisfy the additive composability property in condition~(\ref{eq:add_comp}).

\begin{lemma}
The transformation $\mathsf{CS}$ satisfies the additive composability property, i.e., $\forall x \in M, \theta_1, \theta_2 \in \mathbb{R}^3$,
\[\mathsf{CS}(\mathsf{CS}(x, \theta_1), \theta_2) = \mathsf{CS}(x, \theta_1 + \theta_2).\]
\end{lemma}

\begin{proof}
Let $x = \{(r, g, b)_{ij} \}^{H \times W}, x' = \{(r', g', b')_{ij} \}^{H \times W} = \mathsf{CS}(x, \theta_1)$ and $x'' = \{(r'', g'', b'')_{ij} \}^{H \times W} = \mathsf{CS}(x', \theta_2)$. We need to show that $x'' = \mathsf{CS}(x, \theta_1 + \theta_2)$.
Let $r_{\max}, g_{\max}$ and $b_{\max}$ be the maximum values of the red, green and blue channels respectively of $x$ and $r'_{\max}, g'_{\max}$ and $b'_{\max}$ be the same for $x'$.
From the definition of $\mathsf{CS}$ in section~\ref{sec:color-shift}, we have:
\begin{align*}
r'_{ij} &= \frac{2^{\theta_1^R} r_{ij} }{\mathsf{MAX}}, \quad g'_{ij} = \frac{2^{\theta_1^G} g_{ij} }{\mathsf{MAX}}, \quad b'_{ij} = \frac{2^{\theta_1^B} b_{ij} }{\mathsf{MAX}}\\
\text{and} \quad r''_{ij} &= \frac{2^{\theta_2^R} r'_{ij} }{\mathsf{MAX'}}, \quad g''_{ij} = \frac{2^{\theta_2^G} g'_{ij} }{\mathsf{MAX'}}, \quad b''_{ij} = \frac{2^{\theta_2^B} b'_{ij} }{\mathsf{MAX'}}
\end{align*}
where $\mathsf{MAX} = \max (2^{\theta_1^R} r_{\max}, 2^{\theta_1^G} g_{\max}, 2^{\theta_1^B} b_{\max})$ and $\mathsf{MAX'} = \max (2^{\theta_2^R} r'_{\max}, 2^{\theta_2^G} g'_{\max}, 2^{\theta_2^B} b'_{\max})$.
From the definition of $r'_{\max}$, we have:
\[r'_{\max} = \max{r'_{ij}} = \max{ \frac{2^{\theta_1^R} r_{ij} }{\mathsf{MAX}} } = \frac{2^{\theta_1^R} \max{r_{ij}} }{\mathsf{MAX}} = \frac{2^{\theta_1^R} r_{\max}}{\mathsf{MAX}}.\]
Similarly,
\[g'_{\max} = \frac{2^{\theta_1^G} g_{\max}}{\mathsf{MAX}} \quad \text{and} \quad b'_{\max} = \frac{2^{\theta_1^B} b_{\max}}{\mathsf{MAX}}.\]
Therefore,
\[\mathsf{MAX'} = \frac{\max (2^{\theta_1^R + \theta_2^R} r_{\max}, 2^{\theta_1^G + \theta_2^G} g_{\max}, 2^{\theta_1^B + \theta_2^B} b_{\max})}{\textsf{MAX}}.\]
Substituting $r'_{ij}$ and $\textsf{MAX'}$ in the expression for $r''_{ij}$, we get:
\[r''_{ij} = \frac{2^{\theta_2^R} 2^{\theta_1^R} r_{ij} }{\mathsf{MAX'} \mathsf{MAX}} = \frac{2^{\theta_1^R + \theta_2^R} r_{ij} }{\max (2^{\theta_1^R + \theta_2^R} r_{\max}, 2^{\theta_1^G + \theta_2^G} g_{\max}, 2^{\theta_1^B + \theta_2^B} b_{\max})}.\]
Similarly,
\[g''_{ij} = \frac{2^{\theta_1^G + \theta_2^G} g_{ij} }{\max (2^{\theta_1^R + \theta_2^R} r_{\max}, 2^{\theta_1^G + \theta_2^G} g_{\max}, 2^{\theta_1^B + \theta_2^B} b_{\max})} \quad \text{and} \quad b''_{ij} = \frac{2^{\theta_1^B + \theta_2^B} b_{ij} }{\max (2^{\theta_1^R + \theta_2^R} r_{\max}, 2^{\theta_1^G + \theta_2^G} g_{\max}, 2^{\theta_1^B + \theta_2^B} b_{\max})}.\]
Hence, $x'' = \mathsf{CS}(x, \theta_1 + \theta_2)$.
\end{proof}

\begin{lemma}
The transformation $\mathsf{SV}$ satisfies the additive composability property, i.e., $\forall x \in M, \theta_1, \theta_2 \in \mathbb{R}_{\geq 0}^2$,
\[\mathsf{SV}(\mathsf{SV}(x, \theta_1), \theta_2) = \mathsf{SV}(x, \theta_1 + \theta_2).\]
\end{lemma}

\begin{proof}
Let $x = \{(h, s, v)_{ij} \}^{H \times W}, x' = \{(h, s', v')_{ij} \}^{H \times W} = \mathsf{SV}(x, \theta_1)$ and $x'' = \{(h, s'', v'')_{ij} \}^{H \times W} = \mathsf{SV}(x', \theta_2)$ in HSV format. We need to show that $x'' = \mathsf{SV}(x, \theta_1 + \theta_2)$.
Let $s_{\mean}, s_{\max}, v_{\mean}$ and $v_{\max}$ be the means and maximums of the saturation and brightness values of $x$ and $s'_{\mean}, s'_{\max}, v'_{\mean}$ and $v'_{\max}$ be the same for $x'$.
From the definition of $\mathsf{SV}$ in section~\ref{sec:brightness-saturation}, we have:
\begin{align*}
s'_{ij} = \frac{s_{ij} + (2^{\theta_1^S} - 1) s_{\mean}}{\mathsf{MAX}}, \quad v'_{ij} = \frac{v_{ij} + (2^{\theta_1^V} - 1) v_{\mean}}{\mathsf{MAX}}\\
\text{and} \quad s''_{ij} = \frac{s'_{ij} + (2^{\theta_2^S} - 1) s'_{\mean}}{\mathsf{MAX'}}, \quad v''_{ij} = \frac{v'_{ij} + (2^{\theta_2^V} - 1) v'_{\mean}}{\mathsf{MAX'}}
\end{align*}
where $\mathsf{MAX} = \max(s_{\max} + (2^{\theta_1^S} - 1) s_{\mean}, v_{\max} + (2^{\theta_1^V} - 1) v_{\mean})$ and $\mathsf{MAX'} = \max(s'_{\max} + (2^{\theta_2^S} - 1) s'_{\mean}, v'_{\max} + (2^{\theta_2^V} - 1) v'_{\mean})$.
From the definitions of $s'_{\mean}$ and $s'_{\max}$, we have:
\begin{align*}
s'_{\mean} &= \mean{s'_{ij}} = \mean{ \frac{s_{ij} + (2^{\theta_1^S} - 1) s_{\mean}}{\mathsf{MAX}} } = \frac{\mean{s_{ij}} + (2^{\theta_1^S} - 1) s_{\mean}}{\mathsf{MAX}} = \frac{2^{\theta_1^S} s_{\mean}}{\mathsf{MAX}}\\
s'_{\max} &= \max{s'_{ij}} = \max{ \frac{s_{ij} + (2^{\theta_1^S} - 1) s_{\mean}}{\mathsf{MAX}} } = \frac{\max{s_{ij}} + (2^{\theta_1^S} - 1) s_{\mean}}{\mathsf{MAX}} = \frac{s_{\max} + (2^{\theta_1^S} - 1) s_{\mean}}{\mathsf{MAX}}.
\end{align*}
Similarly,
\[v'_{\mean} = \frac{2^{\theta_1^V} v_{\mean}}{\mathsf{MAX}} \quad \text{and} \quad v'_{\max} = \frac{v_{\max} + (2^{\theta_1^V} - 1) v_{\mean}}{\mathsf{MAX}}.\]
Therefore,
\begin{align*}
    \mathsf{MAX'} &= \max(s'_{\max} + (2^{\theta_2^S} - 1) s'_{\mean}, v'_{\max} + (2^{\theta_2^V} - 1) v'_{\mean})\\
    &= \max(\frac{s_{\max} + (2^{\theta_1^S} - 1) s_{\mean} + (2^{\theta_2^S} - 1) 2^{\theta_1^S} s_{\mean}}{\mathsf{MAX}}, v'_{\max} + (2^{\theta_2^V} - 1) v'_{\mean})\\
    &= \max(\frac{s_{\max} + (2^{\theta_1^S + \theta_2^S} - 1) s_{\mean}}{\mathsf{MAX}}, v'_{\max} + (2^{\theta_2^V} - 1) v'_{\mean})\\
    &= \max(s_{\max} + (2^{\theta_1^S + \theta_2^S} - 1) s_{\mean}, v_{\max} + (2^{\theta_1^V} - 1) v_{\mean} + (2^{\theta_2^V} - 1) 2^{\theta_1^V} v_{\mean})/\mathsf{MAX}\\
    &= \max(s_{\max} + (2^{\theta_1^S + \theta_2^S} - 1) s_{\mean}, v_{\max} + (2^{\theta_1^V + \theta_2^V} - 1) v_{\mean})/\mathsf{MAX}.
\end{align*}
Substituting $s'_{ij}, s'_{\mean}$ and $\mathsf{MAX'}$ in the expression for $s''_{ij}$, we get:
\begin{align*}
    s''_{ij} &= \frac{s_{ij} + (2^{\theta_1^S} - 1) s_{\mean} + (2^{\theta_2^S} - 1) 2^{\theta_1^S} s_{\mean}}{\mathsf{MAX'} \mathsf{MAX}}\\
    &= \frac{s_{ij} + (2^{\theta_1^S + \theta_2^S} - 1) s_{\mean}}{\max(s_{\max} + (2^{\theta_1^S + \theta_2^S} - 1) s_{\mean}, v_{\max} + (2^{\theta_1^V + \theta_2^V} - 1) v_{\mean})}.
\end{align*}

Similarly,
\[v''_{ij}= \frac{v_{ij} + (2^{\theta_1^V + \theta_2^V} - 1) v_{\mean}}{\max(s_{\max} + (2^{\theta_1^S + \theta_2^S} - 1) s_{\mean}, v_{\max} + (2^{\theta_1^V + \theta_2^V} - 1) v_{\mean})}.\]
Hence, $x'' = \mathsf{SV}(x, \theta_1 + \theta_2)$.
\end{proof}

\section{Details for Plots in Figure~\ref{fig:undefVScert}}
\label{sec:gen-nat-dist}
The distribution shifts used to evaluate the empirical performance of the base models in figure~\ref{fig:undefVScert} have been generated by first sampling an image $x$ from the original distribution $\mathcal{D}$ and then randomly transforming it images from the original distribution by adding a noise in the corresponding transformation space.
The Wasserstein bound of these shifts can be calculated by computing the expected perturbation size of the smoothing distribution.
For example, the expected $\ell_2$-norm of a 3-dimensional Gaussian vector is given by $2 \sqrt{2} \sigma / \sqrt{\pi}$ and expected $\ell_1$-norm a 2-dimensional vector sampled uniformly from $[0, b]^2$ is $b$.

The training and smoothing noise levels used for color shift, hue shift and SV shift are (0.8, 10.0), ($180 \degree$, $180 \degree$) and (8.0, 12.0) respectively.

\section{Hue Shift}
\label{sec:apx-hue-plots}

Define a hue shift of an $H \times W$ sized image $x$ by an angle $\theta \in [-\pi, \pi]$
in the HSV space as:
\begin{align*}
\mathsf{HS}(x, \theta) &= \left\{\left(w(h + \theta), s, v \right)_{ij} \right\}^{H \times W}\\
\text{where} \quad w(x) &= x - 2 \pi \left\lfloor \frac{x}{2 \pi} \right\rfloor
\end{align*}
which rotates each hue value by an angle $\theta$ and wraps it around to the $[0, 2\pi)$ range.
It is easy to show that this transformation satisfies additive composability in condition~(\ref{eq:add_comp}).
The Wasserstein distance is defined using the corresponding distance function $d_{\mathsf{HS}}$ by taking the absolute value of the hue shift $|\theta|$.

\begin{lemma}
The transformation $\mathsf{HS}$ satisfies the additive composability property, i.e., $\forall x \in M, \theta_1, \theta_2 \in [-\pi, \pi]$,
\[\mathsf{HS}(\mathsf{HS}(x, \theta_1), \theta_2) = \mathsf{HS}(x, \theta_1 + \theta_2).\]
\end{lemma}

\begin{proof}
Let $h$ be the hue value of the $(i, j)$th pixel of the image $x$.
Since the transformation only affects the hue values, we ignore the other coordinates.
The hue value after the transformation $\mathsf{HS}(\mathsf{HS}(x, \theta_1), \theta_2)$ is given by
\begin{align*}
w(w(h + \theta_1) + \theta_2) &= w \left( h+\theta_1 - 2\pi \left\lfloor \frac{h+\theta_1}{2\pi} \right\rfloor + \theta_2\right)
\end{align*}
\end{proof}

Define a smoothing distribution that applies a random hue rotation $\delta$ sampled uniformly from the range $[-\pi, \pi]$.
Since $\mathsf{HS}$ wraps the hue values around in the interval, the distributions of $h + \delta$ and $(h + \theta) + \delta$ for two hue values shifted by an angle $\theta$ are both uniform in $[0, 2 \pi]$.
Thus, the smoothing distribution for two hue shifted images is the same which implies that $\psi(d(x_1, x_2)) = 0$ whenever $d(x_1, x_2)$ is finite.
Hence, from Theorem~\ref{thm:dist-robust}, we have $\mathbb{E}_{(x_2, y_2) \sim \mathcal{\tilde{D}}} [\bar{h}(x_2, y_2)] \geq \mathbb{E}_{(x_1, y_1) \sim \mathcal{D}} [\bar{h}(x_1, y_1)]$ for hue shifts.
Since, the certified accuracy remains constant with respect to the Wasserstein distance of the shift, we just plot the certified accuracies obtained by various base models trained under different noise levels in figure~\ref{fig:cert-acc-HS}.
We plot the certified accuracies obtained by various models trained using random hue rotations picked uniformly from the range $[-\beta, \beta]$ for different values of the maximum angle $\beta$ in the range.
The certified accuracy roughly increases with the training noise achieving a maximum of 87.9\% for a max angle $\beta = 180 \degree$ for the training noise level.

\begin{figure}[t]
    \centering
    \includegraphics[width=0.5\textwidth, trim={5mm 2mm 12mm 6mm}, clip]{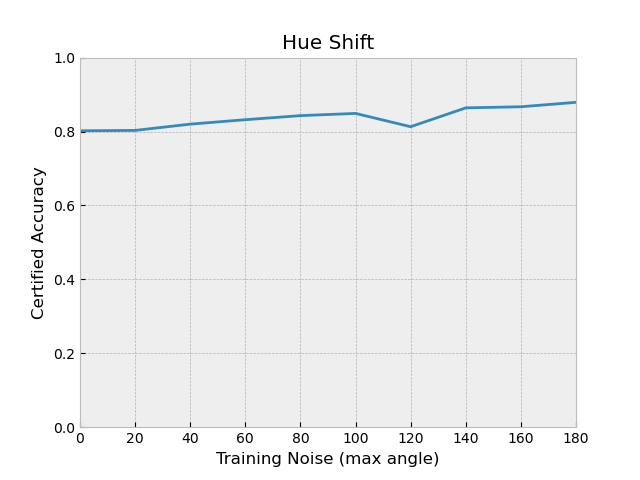}
    \caption{Certified accuracy under hue shift for different levels of training noise. Since, the certified accuracy remains constant with respect to the Wasserstein distance ($\epsilon$) of the shifted distribution, we plot the certified accuracy of models trained with different noise levels $\beta$.}
    \label{fig:cert-acc-HS}
\end{figure}

\section{Random Channel Selection}
\label{sec:random-channel}
Consider a smoothing distribution that randomly picks one of the RGB channels with equal probability, scales it so that the maximum pixel value in that channel is one and sets all the other channels to zero.
This smoothing distribution is invariant to the color shift transformation $\mathsf{CS}$ and thus, satisfies $\psi(d_{\mathcal{T}}(x_1, x_2)) = 0$ whenever $d_{\mathcal{T}}(x_1, x_2)$ is finite.
Therefore, from theorem~\ref{thm:dist-robust}, we have $\mathbb{E}_{z \sim \mathcal{\tilde{D}}} [\bar{h}(z)] \geq \mathbb{E}_{x \sim \mathcal{D}} [\bar{h}(x)]$ under this smoothing distribution for all Wasserstein bounds $\epsilon$ with respect to $d_{\mathsf{CS}}$.
Figure~\ref{fig:random_channel} plots the certified accuracies, using random channel selection for smoothing, achieved by models trained using Gaussian distributions of varying noise levels in the transformation space.
The certified accuracy roughly increases with the training noise achieving a maximum of 87.1\% for a training noise of 0.8.

\section{Experiment Details for Section \ref{sec:adv-attacks}}
As mentioned, for the certified models, we use the released pre-trained ResNet110 models from \cite{cohen19}, using the same level of Gaussian Noise for training and testing. For empirical results, we use the implementation of the $\ell_2$ Carlini and Wagner \cite{Carlini017} attack provided by the IBM ART package \cite{art2018} with default parameters (except for batch size which we set at 256 to increase processing speed.)

\begin{figure}[t]
    \centering
    \includegraphics[width=0.5\textwidth, trim={5mm 2mm 12mm 6mm}, clip]{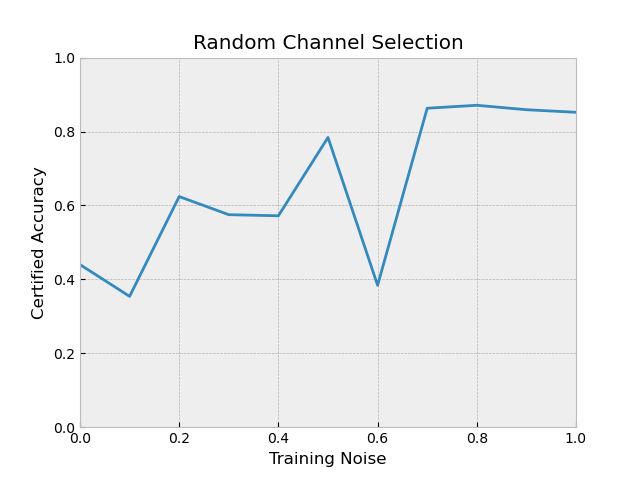}
    \caption{Certified robustness against color shift using random channel selection as the smoothing distribution. Since, the certified accuracy remains constant with respect to the Wasserstein distance ($\epsilon$) of the shifted distribution, we plot the certified accuracy of models trained with various levels of Gaussian noise in the transformation space.}
    \label{fig:random_channel}
\end{figure}

\section{Experiment details for Section \ref{sec:unlearnability}}
Our experimental setting is adapted from the ``sample-wise perturbation'' CIFAR-10 experiments in \cite{huang2021unlearnable}: hyperparameters are the same as in that work unless otherwise stated. 
For background, \cite{huang2021unlearnable} creates an unlearnable dataset by performing the following ``bi-level'' minimization, to simultaneously train a proxy classifier model and create unlearnable examples:
\begin{equation} \label{eq:min_min_attack}
    \min_\theta \min_ {(\epsilon_1, ... ,\epsilon_n)} \frac{1}{n} \sum_{i=1}^n \mathcal{L} (f_\theta(x_i + \epsilon_i), y_i)
\end{equation}
In other words, in contrast with standard training, both the samples and the proxy classifier are optimized to decrease the loss. New classifiers trained on the resulting samples fail to generalize to unperturbed samples. In the experiments, as in \cite{huang2021unlearnable}, the inner minimization over perturbations is performed for 20 steps over the entire dataset, for every one batch update step of the outer minimization. Training stops when training accuracy reaches a threshold value of 99\%.

We now detail differences in experimental setup from \cite{huang2021unlearnable}:
\subsection{Adaptation to $\ell_2$ attack setting}
After each optimization step, we project $\epsilon_i$'s into an $\ell_2$ ball (of radius given by the Wasserstein bound $\epsilon$) rather than an $\ell_\infty$ ball. We also use an $\ell_2$ PGD step:
\begin{equation} 
    \epsilon_i' =  \epsilon_i + \tau \frac{{\nabla}_{\epsilon_i} \mathcal{L}(\cdot)}{\|{\nabla}_{\epsilon_i} \mathcal{L}(\cdot)\|_2}
\end{equation}
Step size $\tau$ was set as 0.1 times the total $\ell_2$ $\epsilon$ bound.
\subsection{Adaptation to offline setting}
As discussed in the test, we modify the algorithm such that the simultaneous training of the proxy model and generation of perturbations does not introduce statistical dependencies between perturbed training samples. This is especially important because, if the victim later makes a train-validation split, this would introduce statistical dependencies  between training and validation samples, making it hard to generalize certificates to a test set.

To avoid this, we construct four data splits:
\begin{itemize}
\item Test set (10000 samples): The original CIFAR-10 test set. Never perturbed, only used in final model evaluation.
\item Proxy training set (20000 samples): Used for the optimization of the proxy classifier model parameters $\theta$ in Equation \ref{eq:min_min_attack} and discarded afterward.
\item Training set (20000 samples): Perturbed using one round of the the standard 20 steps of the inner optimization of Equation  \ref{eq:min_min_attack}, while keeping $\theta$ fixed.
\item Validation set (10000 samples): Perturbed using the same method as the ``Training set.''
\end{itemize}
The victim model is trained on the ``Training Set'' and evaluated on the ``Validation set'' and ``Test set''. We also tested on the clean (unperturbed) version of the validation set.
\subsection{Adaptive attack setting}
When testing our smoothing algorithm, we tested two types of attacks:
\begin{itemize}
    \item Non-adaptive attack: the proxy model is trained and perturbations are generated using undefended models without smoothing: only the victim policy applies smoothing noise during training and evaluation.
    \item Adaptive attack: In the minimization of Equation \ref{eq:min_min_attack}, the loss term $\mathcal{L} (f_\theta(x_i + \epsilon_i), y_i)$ is replaced by the expectation:
    \begin{equation}
       \mathop \mathbb{E}_{\delta\sim\mathcal{N}(0,\sigma^2I)} \mathcal{L} (f_\theta(x_i + \epsilon_i + \delta), y_i)
    \end{equation}
     In other words, this models the expectation of a \textit{smoothed} model, like the victim classifier. This smoothed optimization is used in both the proxy model training, as well as the generation of the training and validation sets. Following \cite{SalmanLRZZBY19}, which proposed a similar adaptive attack for sample-wise smoothed classifiers we approximate the expectation using a small number of random perturbations, which are held fixed for the 20 steps of the inner optimization. In our experiments, we use 8 samples for approximation. Because, at large smoothing noises, this makes the attack much less effective, we cut off training after 20 steps of the outer maximization, rather than relying on the accuracy to reach 99\%. (the maximum number of steps required to converge we observed for the  non-adaptive attack was 15).
\end{itemize}
\subsection{Results}
Complete experimental results are presented in Figure \ref{fig:unlearnability_appdx}. All results are means of 5 independent runs, and error bars represent standard errors of the means.
\begin{figure}
    \centering
    \includegraphics[height=\textheight]{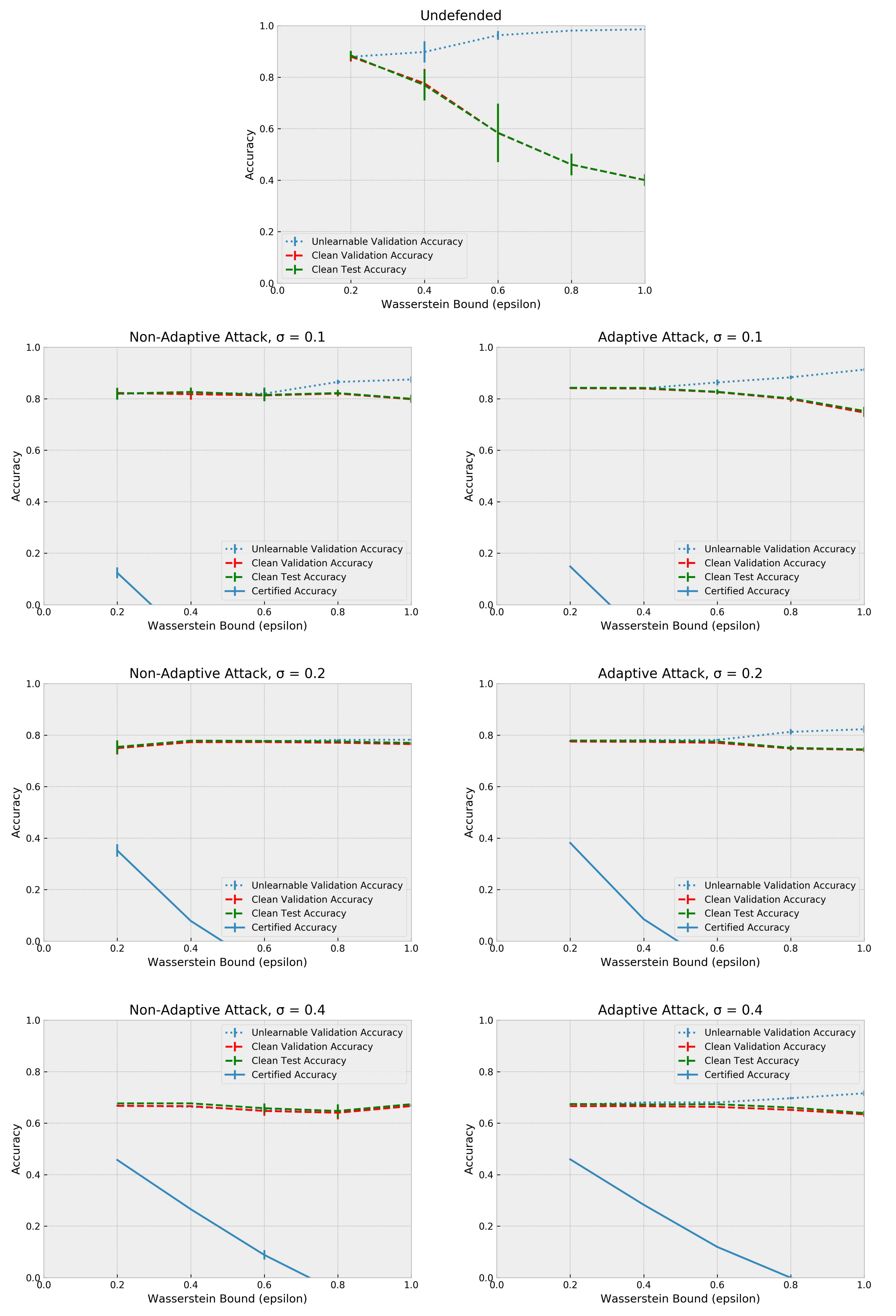}
    \caption{Complete Experimental results for unlearnability experiments.}
    \label{fig:unlearnability_appdx}
\end{figure}
\newpage
\section{Empirical Attacks on $\ell_2$-distributional robustness.}
\label{sec:emp_attacks_l2}
In this section, we describe an empirical attack on $\ell_2$-distributional smoothing. Our attack is based on the attack from \cite{SalmanLRZZBY19}, and we use the code for PGD attack against smoothed classifiers from that work as a base, but there are a few considerations we must make.

\begin{figure}
    \centering
    \includegraphics[width=\textwidth]{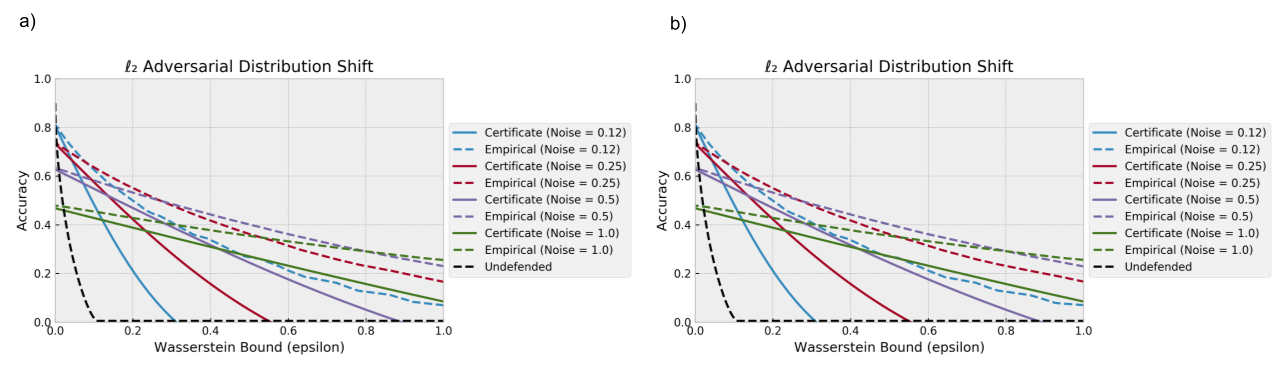}
    \caption{Adversarial attack on distributionally-smoothed classifiers, for CIFAR-10. For smoothed classifiers, we us the PGD attack described  in this section; see main text for details on the  baseline. The dashed lines represent an upper bound on the empirical Wasserstein distances. In plot (a), we use the loss function in Equation \ref{eq:loss_a}, while in (b) we use Equation \ref{eq:loss_b}. }
    \label{fig:adv_attack_smoothed}
\end{figure}

First, while the goal of the attacker in \cite{SalmanLRZZBY19} is to change the output of a classifier that uses the \textit{expected} logits, the goal in our case is to instead reduce the average classification accuracy of \textit{each noise instance}. Concretely, \cite{SalmanLRZZBY19} uses an attacker loss function for each sample $x, y$  of the following form:

    \begin{equation}
        \max_\epsilon \mathcal{L}_{\text{Cross Ent.}} \left( \mathop \mathbb{E}_{\delta\sim\mathcal{N}(0,\sigma^2I)}[\tilde{f}_\theta(x + \epsilon + \delta)], y\right)\label{eq:loss_a}
    \end{equation} 
Where we use $\tilde{f}$ to represent the SoftMax-ed logit function. However, because in our case, the classifier under attack is \textit{not} $\mathbb{E}_{\delta\sim\mathcal{N}(0,\sigma^2I)}[\tilde{f}_\theta(x + \epsilon + \delta)]$, but rather $\tilde{f}_\theta(x + \epsilon + \delta)$ itself, we instead considered the loss function: 
    \begin{equation}
        \max_\epsilon \mathop \mathbb{E}_{\delta\sim\mathcal{N}(0,\sigma^2I)} \left[ \mathcal{L}_{\text{Cross Ent.}} \left( \tilde{f}_\theta(x + \epsilon + \delta), y\right) \right] \label{eq:loss_b}
    \end{equation} 
Empirically, we find the choice of loss function to make very little difference: see Figure \ref{fig:adv_attack_smoothed}. 

We also must consider how to correctly make the attacker ``strategic'': that is, how to allocate attack magnitude so as to attack most effectively while minimizing Wasserstein distance. This is more difficult than in the undefended case, because it is no longer true that for each sample $x$, we can identify the magnitude $\|CW(x, y; g) - x\|_2$ such that an attack of this magnitude is guaranteed to be successful, while a smaller attack is unsuccessful and hence is not attempted. Rather, for a given attack magnitude, there is instead a \textit{probability of success}, over the distribution of $\delta$.

In order to deal with this, we perform PGD at a range of attack magnitudes, specifically $E = \{i/8| i  \in \{1,...,16\}\}$.
Let $PGD_e(x,y;g)$ be the result of the attack at magnitude $e\in E$.
We then define the adaptive attacker as:

\begin{equation}
    \text{Adv}_\gamma(x) := PGD_{e*}(x,y;g)
\end{equation}
Where:
\begin{equation}
\begin{split}
     &e* := \max e \in E \text{ such that  }  \\
     &\frac{\mathop \mathbb{E}_{\delta} \left[ \mathcal{L}_{0/1} \left( \tilde{f}_\theta( PGD_{e}(x,y;g) + \delta), y\right) \right]  - \mathop\mathbb{E}_{\delta} \left[ \mathcal{L}_{0/1} \left( \tilde{f}_\theta( x + \delta), y\right) \right]}{e} > \gamma
\end{split} \label{eq:efficent_attack_smoothed}
\end{equation}
In other words, we use the largest attack such that the \textit{increase in misclassification rate per unit attack magnitude} is above the threshold $\gamma$. If this is not the case for any $e\in E$, we elect not to attack, and set $\text{Adv}_\gamma(x) := x$. As was described in the main text for the baseline case, we sweep over a range of threshold values $\gamma$ when reporting results. When evaluating the  expectations in Equation \ref{eq:efficent_attack_smoothed}, we use a sample of 100 noise instances. However, once $e*$ is identified, we then use a \textit{different}  sample of 100 noise instances per training sample $x$ when reporting the final accuracy: this is to de-correlate the attack generation of  $\text{Adv}_\gamma(x)$ with the evaluation of the attack.  (However, noise instances are kept constant over the sweep  of $\gamma$). When reporting results (the upper bounds of empirical Wasserstein distances), we use  $e*$  as an  upper bound on $\|PGD_{e*}(x,y;g) -x\|_2$, rather than using $\|PGD_{e*}(x,y;g) -x\|_2$ directly.

Attack hyperparameters are taken from \cite{SalmanLRZZBY19}: We use 20 attack steps, a step size of $e/10$, and use 128 noise instances when computing gradients. We evaluate using 10\% of the CIFAR-10 test set. 

\end{document}